\documentclass[a4paper,11pt]{article}
\usepackage{stmaryrd}
\usepackage{amsfonts}
\usepackage{bbm}
\usepackage{amscd}
\usepackage{mathrsfs}
\usepackage{amsmath,amsfonts,amsthm,mathrsfs}
\usepackage[dvips]{graphicx}
\usepackage[utf8]{inputenc}
\usepackage[T1]{fontenc}
\usepackage{lmodern}
\usepackage{amssymb}
\usepackage[all]{xy}
\usepackage{nicefrac,mathtools,enumitem}
\usepackage{microtype}
\usepackage{CJK}
\usepackage{amssymb}
\usepackage{epstopdf}
\usepackage{graphicx}
\usepackage{graphics}
\usepackage[T1]{fontenc}
\usepackage[utf8]{inputenc}
\usepackage{authblk}

\numberwithin{equation}{section}
\newtheorem{thm}{Theorem}
\newtheorem{lem}{Lemma}
\newtheorem{cor}{Corollary}
\newtheorem{pro}{Proposition}

\newtheorem{defi}{Definition}

\newcommand {\emptycomment}[1]{}

\newcommand{\be }{\begin{equation}}
	\newcommand{\ee }{\end{equation}}

\newcommand{\Real}{\mathbb R}

\newcommand{\Numb}{\mathbb N}
\newcommand{\huaB}{\mathcal{B}}
\newcommand{\huaS}{\mathcal{S}}
\newcommand{\huaL}{\mathcal{L}}

\newcommand{\huaE}{\mathcal{E}}

\newcommand{\huaY}{\mathcal{Y}}

\newcommand{\huaP}{\mathcal{P}}

\newcommand{\huaH}{\mathcal{H}}

\newcommand{\huaO}{\mathcal{O}}
\newcommand{\huaT}{\mathcal{T}}

\newcommand{\nono}{\nonumber}
\newcommand{\noi}{\noindent}

\allowdisplaybreaks

\newcommand{\f}{\frac}

\newcommand{\ww}{\widetilde}

\newcommand{\mbb}{\mathbb}

\newcommand{\wh}{\widehat}

\def\bea{\begin{eqnarray}}
	\def\eea{\end{eqnarray}}
\def\be{\begin{equation}}
	\def\ee{\end{equation}}
\def\blm{\begin{lem}}
	\def\elm{\end{lem}}
\def\bes{\begin{eqnarray*}}
	\def\ees{\end{eqnarray*}}

\title{Learning Theory of Distribution Regression with  Neural Networks}

\author[1]{Zhongjie Shi\thanks{zhongjie.shi@esat.kuleuven.be}}
\author[2]{Zhan Yu\thanks{mathyuzhan@gmail.com}}
\author[3]{Ding-Xuan Zhou\thanks{dingxuan.zhou@sydney.edu.au}}

\affil[1]{Department of Electrical Engineering, ESAT-STADIUS, KU Leuven, Belgium}
\affil[2]{Department of Mathematics, Hong Kong Baptist University, Hong Kong}
\affil[3]{School of Mathematics and Statistics, University of Sydney, Australia}

\date{January 2022 (first version)}

\begin{document}
	
	\maketitle
	
	\begin{abstract}
		In this paper, we aim at establishing an approximation theory and a learning theory of distribution regression via a fully connected neural network (FNN). In contrast to the classical regression methods, the input variables of distribution regression are  probability measures. Then we often need to perform a second-stage sampling process to approximate the actual information of the distribution. On the other hand, the classical neural network structure requires the input variable to be a vector.   When the input samples are probability distributions, the traditional deep neural network method cannot be directly used and the difficulty arises for distribution regression. A well-defined neural network structure for distribution inputs is intensively desirable. There is no mathematical model and theoretical analysis on neural network realization of distribution regression. To overcome technical difficulties and address this issue, we establish a novel fully connected neural network framework to realize an approximation theory of functionals defined on the  space of Borel probability measures. Furthermore, based on the established functional approximation results, in the hypothesis space induced by the novel FNN structure with distribution inputs, almost optimal learning rates for the proposed distribution regression model up to logarithmic terms are derived via a novel two-stage error decomposition technique.
	\end{abstract}
	
	\noindent {\it Keywords}: Learning theory,  distribution regression, neural networks, approximation rates, learning rates.
	
	
	\section{Introduction}
	Recent years have witnessed a vast number of applications  of deep learning in various fields of science and engineering. Deep learning based on deep neural networks has become a powerful tool to provide various models and structures for many complex learning tasks, taking advantages of the huge improvement of computing power in the era of big data \cite{Goodfellow2016}. The classical fully connected neural network for learning functions of input variable vector $x=(x_i)_{i=1}^d\in\mbb R^d$ with $J$ layers of neurons $\{H^{(k)}: \mbb R^d\rightarrow\mbb R^{d_k}\}_{k=1}^J$ with width $\{d_k\}_{k=1}^J$ is defined iteratively by
	\be
	H^{(k)}(x)=\sigma\left(F^{(k)}H^{(k-1)}(x)-b^{(k)}\right), \ \ \  k=1,2,...,J, \label{nn1}
	\ee
	where $\sigma:\mbb R\rightarrow\mbb R$ is an activation function acting on vectors component-wise, $F^{(k)}$ is a $d_k\times d_{k-1}$ matrix, $b^{(k)}$ is a bias vector, $H^{(0)}(x)=x$ with width $d_0=d$. Starting from 1980s, approximation theory of  neural networks has been well developed in lots of studies, such as \cite{Cybenko1989,Hornik1989,Barron1993,Mhaskar1993,chui1994,Yarotsky2017,Bolcskei2019,zhou2020acha,Lu2021,Daubechies2022}.
	
	The classical theory of deep neural networks circling around the generalization  with FNNs has been walking towards mature within the last few past decades, and the regression problem by using neural networks is always a popular topic.
	There have been several studies on learning rates of classical least square regression with deep neural networks (DNNs), and satisfactory learning rates have been derived \cite{chui2019b,SchmidtHieber2020}.  However,  in the existing work related to the regression analysis with DNNs, the input data are still limited to the Euclidean vectors.  A natural question appears when the data we need to handle are probability measures and the neural network structure defined in aforementioned works is not suitable for handling the distribution data directly.  That is, how can we utilize deep neural networks to handle such distribution data? Solving this problem is desirable, since in recent years there are a lot of applications dealing with distribution data \cite{rp2001,sspg2016}. One of the most popular objects in statistical learning theory is the distribution regression. However, it is still difficult to perform distribution regression in FNNs since an initial model to realize distribution regression with neural networks is still lacking. One of the main difficulties to realize distribution regression by utilizing neural networks is to establish an appropriate network structure induced by FNNs that is fit for distribution inputs. To overcome the limitation of existing neural network techniques and answer the above question, in this paper we provide a novel FNN structure to realize distribution regression and establish a meaningful regression model, finally we  derive the satisfactory  learning rates.


	Recently, there has been some progress in distribution regression which aims at regressing from probability measures to real-valued responses \cite{psrw2013,sgps2015,sspg2016,fgz2020,yhsz2021}. In such problems, in contrast to classical regression problems with vector inputs, we need to handle the inputs consisting of probability measures. Hence the sampling process often involves two stages. In the first stage, the distribution samples are selected from some meta distribution. To further get access to the actual information of the probability measure samples, we need to draw samples from the distribution samples, that is the so called second stage process.  In the literature of distribution regression,  most of  the existing works and their related regression models are based on a kernel mean embedding technique to transform  probability measure samples to the space of their mean embeddings which is compact. In these works, they choose the hypothesis space as the reproducing kernel Hilbert space induced by some Mercer kernel $K$, and use some regularization and integral operator techniques (e.g.\cite{sz2004, sz2007}), some nice learning rates are derived. However, to realize the  learning theory of distribution regression with FNNs, the aforementioned traditional kernel-based techniques break down. One of the main reasons is that the architecture
	of the neural network is essentially different from the kernel based hypothesis space. Furthermore, as we need to handle the distribution data instead of the classical Euclidean vectors in neural networks, the classical neural network structure as in \eqref{nn1} cannot be utilized directly. Hence, some new neural network structures must be established for handling distributions. To this end, we will first define a novel neural network structure with the variable of probability measures  that will be given in Definition \ref{defFNN} below. In the hypothesis space induced by the novel network structure, we establish a distribution regression framework with probability samples drawn from a class of Wasserstein space $(\huaP(\Omega), W_p)$, where $\huaP(\Omega)$ denotes a collection of Borel probability measures defined on a compact space $\Omega\subset\mbb R^d$, $\huaP(\Omega)$ is naturally a subset of the space $\huaB(\Omega)$ consisting of all Borel measures on $\Omega$, and $W_p$  denotes the well-known Wasserstein metric that will be defined later.
	
	
	We first introduce our two-stage distribution regression model where the inputs are distribution samples on the Wasserstein space $(\huaP(\Omega),W_p)$. In our distribution regression model, we assume that the distribution samples cannot be observed directly, the data are generated through a two-stage sampling process. In the first stage of sampling, the dataset $D=\{(\mu_i,y_i) \}_{i=1}^m$ is i.i.d. sampled from a meta Borel distribution $\rho$ on $\mathcal{U} \times \huaY$, where $\mathcal{U}=\huaP(\Omega)$, and $\huaY=\mathbb{R}$ is the output space. In the second stage of sampling, the dataset is $\hat{D}=\{(\{x_{i,j}\}_{j=1}^{n_i},y_i) \}_{i=1}^m$, where $\{x_{i,j} \in \Omega\}_{j=1}^{n_i}$ are i.i.d. sampled from the probability distribution $\mu_i$, which is a first-stage sample. If we use the empirical terminology $\hat{\mu}_i^{n_i}= \f{1}{n_i}\sum_{j=1}^{n_i} \delta_{x_{i,j}}$, then we can also denote the second-stage sample set by
	$$\hat{D}=\{(\hat{\mu}_i^{n_i},y_i) \}_{j=1}^m.$$
	
	The classical fully connected neural network takes a vector as input, which is not suitable for distribution regression. Since we may have samples of different size $n_i$ for different sample distribution $\mu_i$ in the second stage of sampling, if we just vectorize all the $n_i$ samples in the second stage to a vector as the input of a FNN, when the second-stage sample size $n_i$ is different for different first-stage sample $\left(\mu_i,y_i\right)$, the input dimension of the vector for the FNN is also different, thus making the traditional FNN with vector inputs fails in the distribution regression problem. Therefore, we propose a novel general FNN structure which takes a distribution rather than a vector as the input for the learning of distribution regression. In practical settings, this network structure  is able to utilize the empirical distribution
	\be
	\hat{\mu}_i^{n_i}= \f{1}{n_i}\sum_{j=1}^{n_i} \delta_{x_{i,j}} \label{empirical}
	\ee
	of the second-stage sample rather than the vector as the input. For a Borel measure $\mu$ and  an integrable function vector $h:\Omega\rightarrow\mbb R^D$, if we use the vector integral notation
	\begin{equation*}
		\int_{\Omega}h(x)d\mu(x)=\left[\begin{array}{c}
			\int_{\Omega}(h(x))_1d\mu(x) \\
			\int_{\Omega}(h(x))_2d\mu(x) \\
			\vdots \\
			\int_{\Omega}(h(x))_Dd\mu(x)
		\end{array}\right],
	\end{equation*}
	then the FNN scheme for distribution regression is given in the following definition.
	\begin{defi} [FNN for distribution regression]
		\label{defFNN}
		Let $\mu\in\huaP(\Omega)$ be an input distribution, the FNN for distribution regression $\{h^{(j)}: \huaP(\Omega) \rightarrow \mathbb{R}^{d_j} \}_{j=J_1}^{J}$ of type $(J_1,J)$, with depth $J\in\mbb N$ and realizing level $J_1\in\{0,1,...,J\}$, and width $\{d_j\}_{j=1}^J$ is defined iteratively by:
		\be \label{FNNlayer}
		h^{(j)}(\mu) = \left\{ \begin{array}{ll}
			\int_{\Omega} \sigma \left(F^{(J_1)} \sigma \left( \cdots \sigma\left(F^{(1)}x-b^{(1)}\right) \cdots \right)- b^{(J_1)} \right)  d\mu(x) ,& \hbox{if} \ j=J_1, \\
			\sigma\left(F^{(j)}h^{(j-1)}(\mu)-b^{(j)}\right),& \hbox{if} \ j=J_1+1, \dots , J. \end{array} \right.
		\ee
		where $\sigma$ is the ReLU activation function given by $\sigma(u)=\max\{u,0\}$, $F^{(j)} \in \mathbb{R}^{d_j \times d_{j-1}}$ is a full connection matrix with $d_0=d$, and $b^{(j)} \in \mathbb{R}^{d_j}$ is a bias vector.
	\end{defi}
	\noi An output of the above FNN is just a linear combination of the last layer $$c \cdot h^{(J)}(\mu),$$ where $c \in \mathbb{R}^{d_J}$ is a coefficient vector. In practical settings, an input probability measure always appears as an empirical distribution. For example, in practice,  the actual input  always has the form $\hat{\mu}_i^{n_i}= \f{1}{n_i}\sum_{j=1}^{n_i} \delta_{x_{i,j}}$, $i=1,2,...,m$, where $x_{i,j}\in\Omega\subset\mbb R^d$. Then the explicit form of $h^{(J)}(\hat{\mu}_i^{n_i})$ can be represented layer by layer from
	\be
	h^{(j)}(\hat{\mu}_i^{n_i}) = \left\{ \begin{array}{ll}
		\f{1}{n_i}\sum_{j=1}^{n_i} \sigma \left(F^{(J_1)} \sigma \left( \cdots \sigma\left(F^{(1)}x_{i,j}-b^{(1)}\right) \cdots \right)- b^{(J_1)} \right)   ,& \hbox{if} \ j=J_1, \\
		\sigma\left(F^{(j)}h^{(j-1)}(\hat{\mu}_i^{n_i})-b^{(j)}\right),& \hbox{if} \ j=J_1+1, \dots , J. \end{array} \right. \tag{\ref{FNNlayer}\textasteriskcentered}
	\ee
	then the FNN structure becomes more flexible to handle such empirical variables.
	
	If we use $\huaH_{\text{FNN}}$ to denote some hypothesis space induced by the  above FNN structure for probability measures in Definition \ref{defFNN} (the rigorous definition will be given in Section \ref{subsec23}), then a distribution regression scheme can be described as
	\be
	f_{\hat{D},\huaH_{\text{FNN}}}=\arg\min_{f\in \huaH_{\text{FNN}}}\frac{1}{m} \sum_{i=1}^m \left(f(\hat{\mu}_i^{n_i})-y_i \right)^2. \label{profnn}
	\ee

	In this paper, we  first derive rates of approximating by FNNs two classes of composite functionals with  distribution variables. The functional induced by ridge functions which we call  a ridge functional will be considered first. Then the functional induced by a composite function associated with a polynomial $Q$ will be considered. We will construct the approximation for the former functional with a  FNN of $J=2$, $J_1=1$ and then for the latter one with a  FNN of $J=3$, $J_1=2$ that are both involved in the  FNN structure given in Definition \ref{defFNN}.  Then with the help of the functional approximation results and a covering number bound developed for the hypothesis space induced by the FNN structure defined above, we derive  learning rates for the  FNN distribution regression scheme. The covering number argument is in fact based on a crucial fact of the compactness of the hypothesis space, which is also the foundation of the learning theory framework. In this paper, we have rigorously shown that the new-defined hypothesis space with distribution inputs from the Wasserstein space is compact.
	
	In fact, for a general distribution regression scheme without  kernel regularization, developing appropriate theoretical  results on the learning ability   in some  hypothesis space is still an open question  itself.  On the way of investigating the new distribution regression scheme and deriving  learning rates, one novel two-stage error decomposition technique is proposed in this paper. To the best of our knowledge, such a two-stage error decomposition technique appears for the first time in the literature of the learning theory of distribution regression. It overcomes the difficulty of deriving learning rates for the non-kernel based distribution regression scheme for which the existing literature has not considered yet. Almost optimal  learning rates for distribution regression up to logarithmic terms are  derived by utilizing neural networks. Our method of deriving almost optimal learning rates also highlights the importance of the FNN structure we construct in the two-stage sampling distribution regression scheme.

	\noi \textbf{Notations}:  we fix some notations and terminologies related to the  $\|\cdot\|_{\infty}$ norm of function or functional in different settings of this paper. In additional to the notations of $\|\cdot\|_{\infty}$ for common scale vectors or matrices, without specification, for a continuous function (or say a ``functional'') $f$ defined on the compact space $(\huaP(\Omega), W_p)$, we use the notation $\|h\|_{\infty}=\sup_{\mu\in\huaP(\Omega)}|h(\mu)|$.  For a vector of functions $h: \huaP(\Omega)\rightarrow \mbb R^D$ defined on $\huaP(\Omega)$, we use $\|h\|_{\infty}=\max_{1\leq i\leq D}\sup_{\mu\in\huaP(\Omega)}|(h(\mu))_i|$. For a vector of functions $h:\Omega\rightarrow \mbb R^D$, we also use $\|h\|_{\infty}=\max_{1\le i\leq D}\sup_{x\in\Omega}|(h(x))_i|$.

	\section{Main Results}
	We show the capacity of our proposed FNN for distribution regression by deriving the approximation rate of a  class of composite nonlinear functional $f\circ L_G$ with the variable of probability measures and a univariate function $f:\mbb R\rightarrow\mbb R$
	\be \label{truefunctional}
	f\circ L_G:\mu\mapsto	f\left(L_G(\mu)\right)= f\left(\int_{\Omega}G(x)d\mu(x) \right).
	\ee
	The input $\mu$ is a probability distribution defined on a compact set $\Omega\subset\mbb B$, where $\mbb B:= \{x\in \mathbb{R}^d: \lVert x \rVert_2\leq 1\}$ is the unit ball of the Euclidean space $\mbb R^d$, $L_G$ is the inner functional of the composite nonlinear functional $f\circ L_G$ in the form of
	\be
	\nono L_G(\mu)=\int_{\Omega}G(x)d\mu(x)
	\ee
	defined on the  space $\huaP(\Omega)$  of all Borel probability measures on $\Omega$, and $G$ is a function that induces the functional $L_G$. In our setting, since $\Omega$ is compact, the Riesz–Markov–Kakutani representation theorem asserts that the space $\huaB(\Omega)$ of all finite Borel measures on $\Omega$ is the dual of the space $C(\Omega)$ of  continuous functions on $\Omega$. So, $C(\Omega)$ is isometrically embedded into the closed space of the dual space of $\huaB(\Omega)$. Hence, by considering a continuous function $G$ on $\Omega$ as the inducing function of the functional $L_G$, the formulation of $L_G(\mu)$ is naturally meaningful as a continuous functional defined on $\huaP(\Omega)\subset\huaB(\Omega)$. Furthermore, since $\Omega$ is a compact metric space equipped with the standard Euclidean metric, the space $\huaP(\Omega)$ is also a compact metric space under the Wasserstein metric, which ensures that the learning theory architecture of the proposed neural networks for distribution regression is  well-defined. In the following, we will consider two classes of functionals, one is induced by the ridge function in the form of $G(x)=g(\xi\cdot x)$ which is extremely popular in the realm of applications related to neural networks. The other one is induced by a more general composite function in the form of $G(x)=g(Q(x))$ with  a $q$-degree polynomial $Q$.

	\subsection{Rates of approximating ridge functionals}
	\label{subsec21}
	We first consider the case where $G(x)=g(\xi \cdot x)$ is actually a ridge function, with a feature vector $\xi \in \mathbb{R}^d$. We assume $B_{\xi}=\|\xi\cdot x\|_{C(\Omega)}$ and $g\in C^{0,1}[-B_\xi,B_\xi]$, the space of Lipschitz functions on $[-B_\xi,B_\xi]$, with semi-norm $|g|_{C^{0,1}}:= \sup\limits_{x_1 \not= x_2\in [-B_\xi, B_\xi]}\frac{|g(x_1)-g(x_2)|}{\lVert x_1-x_2\lVert_2}$, and $B_g=\lVert g \rVert_{C([-B_\xi,B_\xi])}$, and $f\in C^{0,\beta}[-B_G,B_G]$, the space of Lipschitz-$\beta$ functions on $[-B_G,B_G]$ with $0<\beta\leq1$, with semi-norm $\vert f\vert_{C^{0,\beta}}:= \sup\limits_{x_1 \not= x_2\in [-B_G, B_G]}\frac{|f(x_1)-f(x_2)|}{\lVert x_1-x_2\lVert_2^\beta}$, in which $B_G$ is defined as
	\be
	B_G=B_g+2B_{\xi}|g|_{C^{0,1}} \label{BGr}
	\ee
	and $\lVert f \rVert_\infty=\lVert f \rVert_{C[-B_G,B_G]}$. Then the functional to be approximated is of the form
	\be \label{ridgeform}
	f\circ L_G^{\xi}:\mu\mapsto f\left(L_G^{\xi}(\mu)\right)=f\left(\int_{\Omega} g\left(\xi \cdot x \right)d\mu(x)\right), \ \mu\in\huaP(\Omega).
	\ee
	During the construction of a FNN  to realize an approximation of the functional $f\left(L_G(\mu) \right)$,
	we need to first approximate the inner functional $L_G^{\xi}$ (we call it a ridge functional) induced by the feature vector $\xi\in\mbb R^d$ defined by
	\be
	L_G^{\xi}: \mu \mapsto \int_{\Omega} g\left(\xi \cdot x\right)d\mu(x),
	\ee
	
	Our first result gives  rates of approximating the composite functional (\ref{ridgeform}) by a  two-layer FNN  for distribution regression defined in Definition \ref{defFNN}. The \textit{free parameters} are defined as the implicit training parameters in our approximator construction, which is different from the total training parameters in the classical FNN structure due to the special structure of our construction.
	
	\begin{thm} \label{thmridge}
		Let the functional $f\circ L_G^{\xi}$ be in the ridge form  (\ref{ridgeform}), where $\xi \in \mbb R^d$ is a feature vector, $g \in C^{0,1}[-B_\xi,B_\xi]$, and $f \in C^{0,\beta}[-B_G,B_G]$ for some $0<\beta \leq 1$. Then for any $N \in \mathbb{N}$, there exists a  FNN of type $(1,2)$ and width $2N+3$ having the structure of Definition \ref{defFNN} with $F^{(j)}$, $b^{(j)}$, $j=1,2$, explicitly constructed such that the following approximation rates hold,
		\be
		\inf_{\|c\|_{\infty}\leq\f{4\lVert f \rVert_\infty N}{B_G}}\left\{\sup_{\mu\in \huaP(\Omega)} \left| c \cdot h^{(2)}(\mu)- f(L_G^{\xi}(\mu)) \right|\right\} \leq \f{2B_G^\beta \left| f \right|_{C^{0,\beta}}+(2B_\xi \lvert g \rvert_{C^{0,1}})^\beta\left| f \right|_{C^{0,\beta}}}{N^\beta}.  \label{ridgeapp}
		\ee
		The total number of free parameters in this network is $\mathcal{N}=8N+d+12$.
	\end{thm}
	
	The architecture of the  FNN of type $(1,2)$ we construct actually depends upon the formulation of the ridge functional $f\left(L_G^\xi(\mu)\right)$. The functional $L_G^\xi(\mu)$ includes the procedure of taking the integration w.r.t. the input distribution $\mu$ on the ridge function $g(\xi \cdot x)$, which can be approximated by a shallow net with one hidden layer. Therefore, the  FNN we construct also takes the integration after the first hidden layer, corresponding to the number of hidden layer to be used for approximating ridge functions. Then we utilize another hidden layer outside the integration layer because the  function $f$ composed with the functional $L_G^\xi$ can just be approximated by one hidden layer.
	
	We end this subsection by presenting an interesting example. Recall the Laplace transform of a probability measure $\mu\in\huaP(\Omega)$ is defined as
	\be
	\huaL(\mu)(\xi)=\int_{\Omega}e^{-x\cdot\xi}d\mu(x), \ \ \  \xi\in \mbb R^d. \label{lapdef}
	\ee
	Then the FNN structure can be used to approximate a class of functionals $\huaL_{\xi}$ induced by the  Laplace transform of the elements in $\huaP(\Omega)$ at $\xi$:
	\be
	\huaL_{\xi}:\mu\mapsto \huaL(\mu)(\xi). \label{laplace}
	\ee
	The following corollary describes rates of this approximation.
	\begin{cor}\label{lapcor}
		Let the functional $\huaL_{\xi}$ be defined by \eqref{lapdef} and (\ref{laplace}), $\xi\in\mbb R^d$ is a feature vector. $B_{\xi}=\|\xi\cdot x\|_{C(\Omega)}$. Then for any $N \in \mathbb{N}$, there exists a  FNN of type $(1,2)$ and width $2N+3$ following the structure of Definition \ref{defFNN} with $F^{(j)}$, $b^{(j)}$, $j=1,2$, explicitly constructed such that the following approximation rates hold,
		\be
		\inf_{\|c\|_{\infty}\leq4N}\left\{\sup_{\mu\in \huaP(\Omega)} \left| c \cdot h^{(2)}(\mu)- \huaL_{\xi}(\mu) \right|\right\} \leq \f{2e^{B_{\xi}}+6B_{\xi}e^{B_{\xi}}}{N}.
		\ee
		The total number of free parameters in this network is $\mathcal{N}=8N+d+12$. Furthermore, for any $\xi\in\Omega$, the above bound can be further improved to
		\be
		\inf_{\|c\|_{\infty}\leq4N}\left\{\sup_{\mu\in \huaP(\Omega)} \left| c \cdot h^{(2)}(\mu)- \huaL_{\xi}(\mu) \right|\right\} \leq \f{8e}{N}.
		\ee
	\end{cor}
	\subsection{Rates of approximating composite functionals with polynomial features}
	Recall an important fact on the space ${\mathcal P}_q^h (\Real^d)$ of homogeneous polynomials on $\Real^d$ of degree $q$ from \cite{LinPinkus1993} and \cite{Maiorov1999a} that ${\mathcal P}_q^h (\Real^d)$ has a basis $\{(\xi_k \cdot x)^q\}_{k=1}^{n_q}$ for some vector set $\{\xi_k\}_{k=1}^{n_q} \subset \Real^d \setminus \{0\}$
	and this vector set can even be chosen in such a way that the homogeneous polynomial set $\{(\xi_k \cdot x)^\ell\}_{k=1}^{n_q}$ spans the space ${\mathcal P}_\ell^h (\Real^d)$ for every $\ell \in \{1, \ldots, q-1\}$, where $n_q =\left(\begin{array}{c} d-1 + q \\ q \end{array}\right)$ is the dimension of ${\mathcal P}_q^h (\Real^d)$. Applying this fact to a polynomial $Q$ of degree $q$ yields the following lemma stated in \cite{zhou2018deepdistributed}.
	\blm \label{ridgepolynomialiden}
	Let $d \in\Numb$ and $q \in\Numb$. Then there exists a set $\{\xi_k\}_{k=1}^{n_q} \subset \{\xi \in\Real^d: |\xi| =1\}$ of vectors with $\ell_2$-norm $1$ such that
	for any $Q \in {\mathcal P}_q (\Real^d)$ we can find a set of coefficients $\{\gamma_{k, \ell}^Q: k=1, \ldots, n_q, \ell=1, \ldots, q\} \subset \Real$ such that
	\begin{equation} \label{ridgepolynomial}
		Q  (x) =Q(0) + \sum_{k =1}^{n_q} \sum_{\ell=1}^q \gamma_{k, \ell}^Q (\xi_k \cdot x)^\ell, \qquad x\in \Real^d.
	\end{equation}
	\elm
	\label{subsec22}
	Now we consider a general class of functional $f\left(L_G^Q(\mu) \right)$ to be approximated, where $G$ is of a composite form $G(x)= g\left(Q(x)\right)$ and  $Q$ is a polynomial of degree $q$ on $\Omega$ with the constant $\wh{B}_Q=\lVert Q \rVert_{C(\Omega)}$. For the polynomial $Q$, Lemma \ref{ridgepolynomialiden} indicates that there is a  vector
	\bes
	\gamma_Q= \left[\gamma_{1,1}^Q \ \gamma_{1,2}^Q \ \ldots \ \gamma_{1,q}^Q \ \gamma_{2,1}^Q \ \gamma_{2,2}^Q \ \ldots \ \gamma_{2,q}^Q \ \ldots \ \gamma_{n_q,1}^Q \ \gamma_{n_q,2}^Q \ \ldots \  \gamma_{n_q,q}^Q \right],
	\ees
	such that
	\be
	Q  (x) =Q(0) + \sum_{k =1}^{n_q} \sum_{\ell=1}^q \gamma_{k, \ell}^Q (\xi_k \cdot x)^\ell, \qquad x\in \Real^d.
	\ee
	We are now ready to  define the following constant
	\be
	B_Q=\wh{B}_Q+2q\|\gamma_Q\|_1. \label{BQ}
	\ee
	For the function $g$,  we assume $g\in C^{0,1}[-B_Q,B_Q]$, the space of Lipschitz-$1$ functions on $[-B_Q,B_Q]$, with semi-norm $|g|_{C^{0,1}}$, and $B_g=\lVert g \rVert_{C[-B_Q,B_Q]}$.
	We can then define the second constant
	\be
	B_G=B_g+3B_Q|g|_{C^{0,1}} \label{BG}
	\ee
	for introducing the following Theorem \ref{thmpoly}.
	
	For function $f$, we assume $f\in C^{0,\beta}[-B_G,B_G]$, with semi-norm $\vert f\vert_{C^{0,\beta}}$, and $\lVert f \rVert_\infty=\lVert f \rVert_{C[-B_G,B_G]}$. In this subsection, we aim at approximating the functional
	\be \label{Qform}
	f\circ L_G^Q: \  \mu\mapsto f\left(L_G^Q(\mu)\right)=f\left(\int_{\Omega} g\left(Q(x) \right)d\mu(x)\right), \ \mu\in\huaP(\Omega).
	\ee
	As in subsection \ref{subsec21}, we need to first approximate the inner functional $L_G^Q(\mu)$ on $\huaP(\Omega)$ defined as
	\be
	L_G^{Q}: \mu \mapsto \int_{\Omega} g\left(Q(x)\right)d\mu(x),
	\ee
	and derive the corresponding approximation rates of the composite nonlinear functional $\mu\mapsto f\left(L_G^Q(\mu)\right)$ defined on $\huaP(\Omega)$, with one more hidden layer to be used in the integral part of the FNN comparing with that of the ridge functional in subsection \ref{subsec21}.
	
	\begin{thm} \label{thmpoly}
		Let the functional $f\circ L_G^{Q} $ be in the composite form  (\ref{Qform}), where $Q$ is a polynomial of degree $q$ on $\Omega$, $g \in C^{0,1}[-B_Q,B_Q]$, and $f \in C^{0,\beta}[-B_G,B_G]$ for some $0<\beta \leq 1$. Then for any $N \in \mathbb{N}$, there exists a  FNN of type $(2,3)$ and widths $d_1=n_q(2N+3)$, $d_2=d_3= 2N+3$ following the structure of Definition \ref{defFNN} with $F^{(j)}$, $b^{(j)}$, $j=1,2,3$ explicitly constructed, such that the following approximation rates hold,
		\be
		\inf_{\|c\|_{\infty}\leq\f{4\lVert f \rVert_\infty N}{B_G}}\left\{\sup_{\mu\in \huaP(\Omega)} \left| c \cdot h^{(3)}(\mu)- f(L_G^Q(\mu)) \right|\right\} \leq \f{2B_G^\beta \left| f \right|_{C^{0,\beta}}+(3B_Q \lvert g \rvert_{C^{0,1}})^\beta\left| f \right|_{C^{0,\beta}}}{N^\beta}.  \label{polyapp}
		\ee
		The total number of free parameters in this network is $\mathcal{N}=(2q+10)N+(d+q)n_q+3q+15$.
	\end{thm}
	
	The reason that we construct a  FNN of type $(2,3)$ also depends upon the formulation of the composite functional $f\left(L_G^Q(\mu) \right)$. The number of hidden layers in the integration part is chosen to be two since we actually need a  net with two hidden layers to approximate the composite function $g\left(Q(x)\right)$. In fact, the position of the integration w.r.t. the input distribution $\mu$ for the  FNN we construct should be consistent with that in the true functional, otherwise the derivation of an approximation rate might be unsuccessful.
	
	Specifically, when the polynomial $Q$ in the functional-inducing function $g(Q(x))$ takes the special form of $Q(x)=\Vert x \Vert^2_2=x_1^2+x_2^2+\dots+x_d^2$, then the composite functional-inducing  function $g\left(Q(x)\right)$ becomes the radial function $g\left(\Vert x \Vert_2^2 \right)$ and the functional has the form
	
	\be \label{radialform}
	f\circ L_G^{\Vert \cdot \Vert_2}:\mu \mapsto f\left(L_G^{\Vert \cdot \Vert_2}(\mu)\right)=f\left(\int_{\Omega} g\left(\Vert x \Vert_2^2 \right)d\mu(x)\right).
	\ee
	
	Since we can write
	$$Q(x)=\Vert x \Vert_2^2= \sum_{k=1}^d \left(e_k \cdot x\right)^2$$
	in terms of the standard basis $\{e_k\}_{k=1}^d$ of $\mbb R^d$, also $\wh B_Q=\Vert Q \Vert_{C(\Omega)} \leq 1$,  we have $B_Q=\wh B_Q+ 2q\Vert \gamma_Q \Vert_1 \leq 1+4d$.
	
	We are then able to derive the following rates of approximating $f\circ L_G^{\Vert \cdot \Vert_2}$ by utilizing Theorem \ref{thmpoly} directly.
	\begin{cor} \label{cororadial}
		Let the functional $f\circ L_G^{\Vert \cdot \Vert_2}$ be in the composite form  (\ref{radialform}), where $g \in C^{0,1}[-B_Q,B_Q]$, and $f \in C^{0,\beta}[-B_G,B_G]$, for some $0<\beta \leq 1$. Then for any $N \in \mathbb{N}$, there exists a  FNN of type $(2,3)$ and widths $d_1=(2N+3)d$, $d_2=d_3= 2N+3$ following the structure of Definition \ref{defFNN}, such that the following approximation rates hold,
		\be
		\inf_{\|c\|_{\infty}\leq\f{4\lVert f \rVert_\infty N}{B_G}}\left\{\sup_{\mu\in \huaP(\Omega)} \left| c \cdot h^{(3)}(\mu)- f\left( L_G^{\Vert \cdot \Vert_2}(\mu) \right) \right|\right\} \leq \f{2B_G^\beta \left| f \right|_{C^{0,\beta}}+(12d+3)^{\beta} \lvert g \rvert_{C^{0,1}}^\beta\left| f \right|_{C^{0,\beta}}}{N^\beta}.
		\ee
		The total number of free parameters in this network is $\mathcal{N}=12N+d^2+d+18$.
	\end{cor}
	
	\subsection{Distribution regression with FNN}
	\label{subsec23}
	In this subsection, we conduct generalization analysis of the commonly used empirical risk minimization (ERM) algorithm for distribution regression based on FNNs, which measures the generalization ability of the truncated empirical target function of the ERM algorithm.
	
	We start with a  formulation of the distribution regression model in this paper. First we recall some concepts on Wasserstein metric and the corresponding Wasserstein space. Such a space would serve as the basic underlying space for our regression analysis. Now let $(\Omega,d)$ be a complete, separable metric (Polish) space. Let $p\in[1,\infty)$, for  two probability measures $\mu$ and $\nu$ on $\Omega$, the Wasserstein metric of order $p$  is defined as
	\bea
	\nono W_{p}(\mu, \nu) &=&\left(\inf _{\pi \in \Pi(\mu, \nu)} \int_{\Omega\times\Omega} d(x, y)^{p} d \pi(x, y)\right)^{1 / p} \\
	\nono &=&\inf \left\{\bigg[\mathbb{E} d(X, Y)^{p}\bigg]^{\frac{1}{p}}, \quad \operatorname{law}(X)=\mu, \quad \operatorname{law}(Y)=\nu\right\},
	\eea
	where the infimum is taken over all pairs of random vectors $X$ and $Y$ marginally distributed as $\mu$ and $\nu$ ($X\sim\mu$, $Y\sim\nu$). In above definition, $\pi$ is also called a transference plan from $\mu$ to $\nu$ satisfying $\int_{\Omega}d\pi(x,\cdot)=d\mu(x), \ \int_{\Omega}d\pi(\cdot,y)=d\nu(y)$.
	$\Pi(\mu, \nu)$ is used to denote all such transference plans.
	For $W_p$ defined above, earlier literatures have shown that it satisfies the basic axioms of a metric \cite{villani}. As an classical example, the distance $W_1$, that is often called Kantorovich-Rubinstein metric, has played an important role in recent statistical learning problems \cite{zz2018,zlt2021}.
	
	Now we consider a working space that is convenient for the study of our distribution regression model with FNNs. The space is the so-called Wasserstein space. Denote the collection of all Borel probability measures on $\Omega$ by $\huaP(\Omega)$. Let $x_0$ be a fixed but arbitrarily chosen point in $\Omega$,  the Wasserstein space of order $p$ is defined as
	$$
	\huaP_{p}(\Omega)=\left\{\mu \in \huaP(\Omega): \quad \int_{\Omega} d\left(x_{0}, x\right)^{p} d\mu(x)<+\infty\right\}
	.$$
	This space does not depend on the choice of the point $x_{0}$. Then $W_{p}$ defines a (finite) metric on $\huaP_{p}(\Omega)$.
	
	It can be found in \cite{villani} that, for any fixed $p\in[1,\infty)$, the metric space $(\huaP_p(\Omega), W_p)$ is a compact metric space when $(\Omega,d)$ is a compact subset of $\mbb R^n$ equipped with the standard Euclidean metric $d(x,y)=\|x-y\|_2$ as in aforementioned subsection. On the other hand, it is easy to see that when $d(x,y)=\|x-y\|_2$, for any $\mu\in\huaP(\Omega)$, there holds
	\be
	\nono \int_{\Omega} \left\|x_{0}- x\right\|_2^{p} d\mu(x)<+\infty
	\ee
	when $\Omega$ is compact. Hence, under this setting, we have $\huaP_p(\Omega)=\huaP(\Omega)$ and the metric space $(\huaP(\Omega), W_p)$ is compact. In the following of the paper, we will use the notation that $\Omega$ is a compact space with standard Euclidean metric. For fixed $p\in [1,\infty)$ which will be selected depending on different learning situations,  the distribution regression framework will be considered in such a compact metric space $(\huaP(\Omega),W_p)$.

	Inspired by the two-stage regression scheme for distribution regression, we propose our distribution regression model in Wasserstein space $(\huaP(\Omega),W_p)$. The first-stage data set $D=\{(\mu_i,y_i) \}_{i=1}^m$ are i.i.d. sampled from an unknown meta distribution that is a Borel probability measure $\rho$ on $\mathcal{Z}=\mathcal{U} \times \huaY$, where $\mathcal{U} = (\mathcal{P}(\Omega),W_p)$ is the input metric space of Borel probability measures on $\Omega$ with Wasserstein metric, and $\huaY = [-M,M]$ is the output space with some $M>0$. The regression function $f_\rho$ on $\mathcal{U}$ is defined as
	\be
	f_\rho(\mu)=\int_{\mathcal{Y}} y d\rho(y \vert \mu),
	\ee
	where $\rho(y \vert \mu)$ is the conditional distribution  at $\mu$ induced by $\rho$, and it minimizes the mean squared  error
	$$\mathcal{E}(f)= \int_{\mathcal{Z}} (f(\mu)-y)^2 d \rho.$$
	We denote $\rho_{\mathcal{U}}$ as the marginal distribution of $\rho$ on $\mathcal{U}$, and $\left(L_{\rho_{\mathcal{U}}}^2, \Vert \cdot \Vert_\rho \right)$ as the space of square integrable functions with respect to $\rho_{\mathcal{U}}$. For convenience of later representations in some places, for a continuous function $f$ defined on the compact space $(\huaP(\Omega), W_p)$, we also use the notation
	\be
	\|f\|_{\infty}=\sup_{\mu\in\huaP(\Omega)}|f(\mu)|.  \label{infinitynorm}
	\ee
	
	The hypothesis space we use for the ERM algorithm follows the 3-layer FNN structure constructed in the proof of Theorem \ref{thmpoly} with  a positive constant $R$ depending on $d,Q,f,g$ associated with the fixed functional $f\circ L_G^Q$ to bound the parameters in the FNN given by
	\be \label{hypospace}
	\mathcal{H}_{(2,3),R,N}= \left\{ c \cdot h^{(3)}(\mu): \lVert F^{(j)} \rVert_\infty \leq RN^2, \lVert b^{(j)} \rVert_\infty \leq R, \ \hbox{for} \ j=1,2,3, \lVert c \rVert_\infty \leq RN \right\},
	\ee
	where the $\Vert \cdot \Vert_\infty$ norm for matrices is the maximum value of $\ell_1$-norms of its rows. For the hypothesis space $\mathcal{H}_{(2,3),R,N}$, one of the most fundamental problems is whether it is well-defined as a compact hypothesis space for the learning theory of our distribution regression model? In the next theorem, we give a positive answer by rigorously proving the compactness of space $\mathcal{H}_{(2,3),R,N}$, we show that $\mathcal{H}_{(2,3),R,N}$ can be actually used as the hypothesis space of the distribution regression model \eqref{profnn}. We use $(C(\huaP(\Omega)),\|\cdot\|_{\infty})$ to denote space of continuous functions  on $(\huaP(\Omega),W_p)$ equipped with the infinity norm defined in \eqref{infinitynorm}.
	\begin{thm}\label{cptthm}
		The space $\mathcal{H}_{(2,3),R,N}$ is a compact metric subspace of the space $(C(\huaP(\Omega)),\|\cdot\|_{\infty})$.
	\end{thm}
	
	Recall that for a subset $\mathcal{H}$ of a normed space equipped with a norm $\lVert \cdot \rVert$, the $\epsilon$-covering number $\mathcal{N}(\mathcal{H},\epsilon,\lVert \cdot \rVert)$ is the minimum number of balls with radius $\epsilon>0$ that cover $\mathcal{H}$. After showing the compactness of $\mathcal{H}_{(2,3),R,N}$, the covering number
	$$\mathcal{N}(\mathcal{H}_{(2,3),R,N},\epsilon,\lVert \cdot \rVert_{\infty})$$
	makes sense, is finite and would act as one of the main objects to establish other main results. Estimates of $\mathcal{N}(\mathcal{H}_{(2,3),R,N},\epsilon,\lVert \cdot \rVert_{\infty})$ will be given in subsection \ref{cnestimate}.
	
	Since for the distribution regression model, the concrete information of the first stage distribution samples are unavailable, we can only observe the second-stage data $\hat{D}=\{(\{x_{i,j}\}_{j=1}^{n_i},y_i) \}_{i=1}^m$, where $\{x_{i,j} \in \Omega\}_{j=1}^{n_i}$ are i.i.d. sampled from the probability distribution $\{\mu_i\}_{i=1}^m$, so we have to use the empirical distribution $$\hat{\mu}^{n_i}_i=\frac{1}{n_i} \sum_{j=1}^{n_i} \delta_{x_{i,j}}$$ as the input of our FNN structure. Denote the empirical error with respect to the second-stage data $\hat{D}$ as
	\bes
	\mathcal{E}_{\hat{D}}(f):= \frac{1}{m} \sum_{i=1}^m \left(f(\hat{\mu}_i^{n_i})-y_i \right)^2,
	\ees
	then the empirical target function from the ERM algorithm using the hypothesis space (\ref{hypospace}) is the function in $\mathcal{H}_{(2,3),R,N}$ that minimizes the empirical error:
	\be
	f_{\hat D,R,N}:= \arg\min_{f \in \mathcal{H}_{(2,3),R,N}} \mathcal{E}_{\hat{D}}(f). \label{drscheme}
	\ee
	The compactness of the hypothesis space $\mathcal{H}_{(2,3),R,N}$ proved in Theorem \ref{cptthm}  ensures the existence of a minimizer $f_{\hat D,R,N}$ of the above variational problem, and then the learning theory of the above distribution regression scheme makes sense in such a FNN-inducing space.
	We now define the projection operator $\pi_{M}$  on the space of  functional $f: \huaP(\Omega) \rightarrow \mathbb{R}$ as
	$$
	\pi_{M}(f)(\mu)= \begin{cases}M, & \text { if } f(\mu)>M, \\ -M, & \text { if } f(\mu)<-M, \\ f(\mu), & \text { if }-M \leq f(\mu) \leq M .\end{cases}
	$$
	Since the regression function $f_{\rho}$ is bounded by $M$,  we use the truncated empirical target function $$\pi_{M} f_{\hat D,R,N}$$ as the final estimator.

	We first derive an oracle inequality for the distribution regression scheme \eqref{drscheme}, which is different from those for the traditional one-stage statistical regression or kernel-based two-stage distribution regression. The oracle inequality for the traditional regression framework in \cite[lemma 4]{SchmidtHieber2020} fails in the distribution regression framework, since the empirical target function of the ERM algorithm is learned from the second-stage data rather than the first-stage data. We utilize a novel two-stage error decomposition method for distribution regression by including the empirical error of the first-stage sample
	\bes
	\mathcal{E}_{D}(f):= \frac{1}{m} \sum_{i=1}^m \left(f(\mu_i)-y_i \right)^2
	\ees
	as an intermediate term of the error decomposition. Such a two-stage error decomposition is new in the literature of learning theory. We present a novel error decomposition in the following proposition.
	\begin{pro}\label{errordec}
		Denote $\huaH=\mathcal{H}_{(2,3),R,N}$, then for any $h\in \mathcal{H}$ and $f_{\hat D,R,N}$  defined in \eqref{drscheme},
		\bes
		\begin{aligned}
			& \mathcal{E}\left(\pi_{M}f_{\hat D,R,N} \right)- \mathcal{E}\left(f_\rho \right)\leq \mathcal{E}\left(\pi_{M}f_{\hat D,R,N} \right)- \mathcal{E}_{D}\left(\pi_{M}f_{\hat D,R,N} \right)+ \mathcal{E}_D\left(\pi_{M}f_{\hat D,R,N} \right) \\
			& - \mathcal{E}_{\hat{D}}\left(\pi_{M}f_{\hat D,R,N} \right)
			+\mathcal{E}_{\hat{D}}\left(h \right)- \mathcal{E}_{D}\left(h \right)+ \mathcal{E}_{D}\left(h \right) - \mathcal{E}(h)+ \mathcal{E}(h)- \mathcal{E}(f_\rho),
		\end{aligned}
		\ees
		which can be further bounded by $I_1(D,\mathcal{H})+ I_2(D,\mathcal{H}) + \left\vert I_3(\hat{D},\mathcal{H})\right\vert+ \left\vert I_4(\hat{D},\mathcal{H})\right\vert+ R(\mathcal{H})$, in which
		\bes
		\begin{aligned}
			& I_1(D,\mathcal{H}) = \left\{\mathcal{E}\left(\pi_{M}f_{\hat D,R,N} \right)- \mathcal{E}\left(f_\rho \right) \right\}- \left\{\mathcal{E}_{D}\left(\pi_{M}f_{\hat D,R,N} \right)-\mathcal{E}_{D}\left(f_\rho \right) \right\}, \\
			& I_2(D,\mathcal{H}) = \left\{\mathcal{E}_{D}\left(h \right)-\mathcal{E}_D\left(f_\rho \right)\right\} - \left\{\mathcal{E}(h)-\mathcal{E}\left(f_\rho \right)\right\}, \\
			& I_3(\hat{D},\mathcal{H})= \mathcal{E}_D\left(\pi_{M}f_{\hat D,R,N} \right) - \mathcal{E}_{\hat{D}}\left(\pi_{M}f_{\hat D,R,N} \right), \\
			& I_4(\hat{D},\mathcal{H}) = \mathcal{E}_{\hat{D}}\left(h \right)- \mathcal{E}_{D}\left(h \right), \quad \quad
			R(\mathcal{H})= \mathcal{E}(h)- \mathcal{E}(f_\rho).
		\end{aligned}
		\ees
	\end{pro}
	
	Based on the two-stage error decomposition, we derive a new oracle inequality for the distribution regression scheme using our hypothesis space $\mathcal{H}_{(2,3),R,N}$ in the following theorem.
	
	\begin{thm} \label{oracledr}
		Consider the distribution regression framework described above with the first stage sample size  $m$, and the second stage sample size  $n_1=n_2=\cdots=n_m=n$. Then for any $h \in \mathcal{H}_{(2,3),R,N}$ and $\epsilon>0$, we have
		\bes
		\begin{aligned}
			& \textit{Prob}\left\{\Vert \pi_{M}f_{\hat D,R,N}-f_\rho\Vert_\rho^2>2\left\Vert h-f_\rho\right\Vert_\rho^2+8\epsilon \right\} \\
			\leq & \exp\left\{ T_1 N \log{\frac{16M\wh{R}}{\epsilon}}+T_2N\log{N}-\frac{3m\epsilon}{2048M^2} \right\} \\
			+ & \exp \left\{-\frac{m\epsilon^2}{2\left(3M+ \Vert h\Vert_\infty \right)^2 \left(\left\Vert h-f_\rho\right\Vert_\rho^2+\frac{2}{3}\epsilon \right)}\right\} \\
			+ & \exp\left\{\log{4m}+ T_1 N \log{\frac{80M \wh{R} R^2 N^4}{\epsilon}}+T_2N\log{N}-\frac{n\epsilon^2}{115200\max \{\left\Vert h \right\Vert_\infty^2 ,M^2\} R^8N^{16}}\right\},
		\end{aligned}
		\ees
		where $R$, $\wh{R}$ are constants depending on $d,Q,g,f$, and $T_1$, $T_2$ are constants depending on $d,Q$ given explicitly in the proof.
	\end{thm}
	
	Utilizing the above oracle inequality for distribution regression, we are able to achieve almost optimal learning rates for the proposed distribution regression scheme on the hypothesis space $\mathcal{H}_{(2,3),R,N}$ that includes the FNN structure  we construct.
	\begin{thm} \label{thmestimation}
		If $f_\rho=f\circ L_G^Q$ in the composite form  (\ref{Qform}), where $Q$ is a polynomial of degree $q$ on $\Omega$, $g \in C^{0,1}[-B_Q,B_Q]$, and $f \in C^{0,\beta}[-B_G,B_G]$ for some $0<\beta \leq 1$. If the first stage sample size $m$ satisfies the restriction that
		\bes
		\log{(4m)}\leq A_6 m^{\frac{1}{2\beta+1}},
		\ees
		and the neural network parameter $N$ and the second stage sample size $n$ are chosen by
		\bes
		N=\left[A_4 m^{\frac{1}{2\beta+1}}\right], \quad n\geq\left\lceil A_5 m^{\frac{4\beta+17}{2\beta+1}}\right\rceil,
		\ees
		then there exists a constant $A_7$ such that
		\bes
		\mathbb{E}\left\{ \mathcal{E}\left(\pi_{M}f_{\hat D,R,N} \right)- \mathcal{E}\left(f_\rho \right) \right\} \leq A_7 m^{-\frac{2\beta}{2\beta+1}}\log{m},
		\ees
		where $A_4, A_5, A_6$ and $A_7$ are constants depending on $d,Q,g,f$ given explicitly in the proof.
	\end{thm}
	
	Notice that in the proof of Theorem \ref{oracledr} and Theorem \ref{thmestimation}, we assume that all $n_i=n$ for simplicity, but the proof still works for $n_i$ chosen to be different, in which we just need to replace $n$ by $\min_i \{n_i\}$ in the proof, this also stresses the advantage of using distribution rather than vectors as inputs of FNNs for the distribution regression problem. We end this subsection by providing  an intrinsic fact that the composite functional $f\left( L_G(\mu) \right)$ should be continuous w.r.t. the probability variable $\mu$. We put this fact in the following proposition and the proof will be given in the appendix.
	
	\begin{pro}\label{continuous}
		For any fixed $p\in[1,\infty)$, the composite functional $f\left( L_G(\mu) \right)$ with distribution variable $\mu$ is continuous on the metric space $(\huaP(\Omega),W_p)$.
	\end{pro}

	\section{Related work and discussion}
	In this section, we review some related works on approximation and generalization performance of fully connected neural networks for regression with vector inputs, and some researches on distribution regression, to demonstrate the novelty and superiority of our results.
	
	Dating back to thirty years ago, there have been many works analysing the approximation and generalization ability of fully connected neural networks on learning various classes of functions, wherein the H\"older spaces are one of the most important classes of functions to be studied. At the beginning, most results about the approximation ability of the fully connected neural networks are derived  using the $C^{\infty}$ activation functions that satisfy two assumptions, that is for some $b\in \mathbb{R}$, $\sigma^k(b)\neq 0$ for any non-negative integer $k$, and for some integer $q\neq 1$, $\lim_{u \to -\infty} \sigma(u)/\vert u\vert^q=0$ and $\lim_{u \to \infty} \sigma(u)/u^q=1$. Such approximation rates can be found in \cite{Mhaskar1993} through a localized Taylor expansion method. It was shown that for any $f\in W^{\beta}_\infty\left([-1,1]^d\right)$, there exists a  net $f_N$  such that $\left\Vert f_N-f \right\Vert_{C([-1,1]^d)} \leq c_{f,d,\beta}N^{-\beta/d}$, with a constant $c_{f,d,\beta}$ independent of $N$. ReLU was more efficient and effective in deep learning \cite{Goodfellow2016} and thus became more and more popular in practice, therefore it is desirable to find approximation results for  deep ReLU neural networks.  Such approximation results are derived in \cite{Yarotsky2017} that $f\in W^{\beta}_\infty\left([-1,1]^d\right)$ can be approximated with an error $\epsilon$, by using at most $c\left(\log{1/\epsilon+1}\right)$ layers and at most $c\epsilon^{-d/\beta}\left(\log{1/\epsilon+1}\right)$ weights and computation units. Based on these approximation results, generalization results for ERM algorithms using fully connected neural networks are then conducted in \cite{chui2019b} and \cite{SchmidtHieber2020} to achieve  optimal convergence rate in the traditional regression framework. However, these research results are all for the traditional regression problem, the vector-input FNN does not work for the distribution regression framework with two stages of sampling due to the fact that the different sizes of second-stage sample will result in different dimensions of input vectors if we just vectorize all the second-stage samples. Furthermore, the generalization analysis for the traditional regression problem also fails in the distribution regression since the ERM algorithm is based on the second-stage data $\hat{D}$ rather than the first-stage data $D$, resulting in the desire of a new method to bound the concentration error between the generalization error and the second-stage empirical error for the two-stage sampling distribution regression framework.
	
	To solve the limitation of the usage of vector-input neural network structure for the distribution regression framework, we propose a novel FNN structure that is able to take the empirical distribution as the input for the learning of distribution regression in practice. The  idea of constructing the FNN structure with distribution input is slightly related to the work of \cite{Bruna2021}, in which they transform the problem of learning symmetric functions to the functional learning problem with distribution input in a Wasserstein space, and construct a shallow net
	\be \label{brunanet}
	f(\mu)= \frac{1}{m'} \sum_{j'=1}^{m'} b_{j'}\tilde{\sigma}\left(\frac{1}{m}\sum_{j=1}^m c_{j',j}\int \sigma_\alpha\left(\langle w_{j',j},\tilde{x}\rangle \right)d\left(\mu(x)\right) \right), \quad \tilde{x}=[x,R]^T.
	\ee
	The architecture of the FNN for functional approximation and distribution regression we construct (Definition \ref{defFNN}) is more general than (\ref{brunanet}) and can be deeper, which enables us to learn more complex functional. Utilizing our novel FNN structure, we construct a three-layer FNN to derive the approximation rate $O\left(N^{-\beta}\right)$ for learning the functional $f\left(L_G^Q(\mu)\right)$ induced by the polynomial $Q$ and H\"older functions $f, g$, which is the same as the rate of learning the composite function $f\left(g\left(Q(x)\right)\right)$ itself without taking into  consideration  the input of probability measures \cite{SchmidtHieber2020}. Actually, our FNN structure has the potential to learn  functionals that are induced by more composite functions, where we just need to construct a deep FNN with the number of layers determined by the number of composite functions that induce the functionals. Furthermore, we utilize a novel error decomposition method to derive an oracle inequality for the distribution regression framework with two-stage sampling, and then combine it with the approximation rate to derive an almost optimal learning rate $O\left(m^{-\frac{2\beta}{2\beta+1}}\log{m}\right)$ as \cite{chui2019b} and \cite{SchmidtHieber2020}.
	
	Let us review some related works on distribution regression in the literature of learning theory \cite{Cucker2002,sc2008}. The recent main works on the two-stage distribution regression mainly include \cite{sspg2016}, \cite{fgz2020}, \cite{ds2021}, \cite{ds2022}, \cite{mn2021} and \cite{yhsz2021}. These works mainly rely on a mean embedding technique to transform  probability measures to a space consisting of all mean embeddings via some Mercer kernel. In this theoretical model, the learning theory framework is in fact performed on the space of all mean embeddings which form a compact subspace of the  continuous function space defined on the underlying space of the probability measures. To derive nice learning rates, the works of \cite{sspg2016}, \cite{gs2019}, \cite{fgz2020}, \cite{ds2021}, \cite{ds2022} and \cite{yhsz2021} deeply rely on the kernel regularization and the corresponding integral operator technique. In fact, to the best of our knowledge,  learning rate  for distribution regression are  unexplored in  non-kernel or non-regularized settings where the integral operator approach cannot always  be applied. Moveover, most of the classical kernel-based learning theory approaches such as   \cite{sz2007}, \cite{cs2007}, \cite{gs2019}, and \cite{guo2020modeling} fail in the deep learning setting, and such classical techniques are  often inappropriate for modern fully connected neural network structures. One of the main reasons is that the structure of FNN-based hypothesis spaces is totally different from that of the kernel-based hypothesis space. Hence,  we are mainly in face of three aspects of difficulties in this work. The first one is how to provide an appropriate FNN structure with distribution inputs and a meaningful distribution regression model for the hypothesis space induced by the FNN. The second one is how to establish approximation results without utilizing the traditional kernel methods and its related integral operator approaches that the existing works on distribution regression used. The third one is how to realize  distribution regression without using regularization in a novel hypothesis space that relies on FNN structures and derive optimal learning rates, since there is currently no theoretical result on distribution regression derived in such a setting.
	
	In this paper, we overcome  aforementioned difficulties and establish a  learning theory framework of distribution regression with neural networks. At first, some new functional approximation results are proposed. Then we provide the proof of the crucial fact that the space  $\mathcal{H}_{(2,3),R,N}$ is compact. In contrast to the existing works on distribution regression, the underlying hypothesis space defined by us possesses  nice FNN learning structures. Furthermore, though the classical effective dimension based analysis related to kernel methods in \cite{sspg2016}, \cite{fgz2020}, \cite{mn2021} and \cite{yhsz2021} fails in our setting, new estimates based on the covering number of the newly defined hypothesis space $\mathcal{H}_{(2,3),R,N}$ are obtained. In contrast to the  previous works utilizing the classical error decomposition for kernel based regression such as some earlier works \cite{Cucker2002,wyz2006,Cucker2007} and recent developments \cite{fhsys2015,lin2018distributed}, another technical novelty is the crucial two-stage error decomposition which appears first in the literature of learning theory. Based on the functional approximation results, covering number estimates and the two-stage error decomposition, almost optimal learning rates of the proposed FNN distribution regression scheme are obtained. To the best of our knowledge, the work provides a beginning of the approximation theory of functionals with variable of probability measures using neural networks. It is also the first work to  establish the learning theory of distribution regression and obtain almost optimal learning rates with neural networks. The methods in this work have the potential to be further applied to more general deep learning settings.

	\section{Proof of  Main Results}
	
	\subsection{Proof of Theorem \ref{thmridge}}	
	In the construction of the  FNN to approximate the  functional, we need to use the following lemma to approximate  Lipschitz continuous univariate functions using  continuous piecewise linear functions (splines) $\{\sigma(\cdot - t_i)\}_{i=1}^{2N+3}$, with $t_i=-1+\frac{i-2}{N}$, which can be found in \cite{zhou2018deepdistributed} and \cite{msz2021}.
	
	\blm \label{splines}
	For $N\in\mathbb{N}$, let $\textbf{t}=\left\{t_i: = -1+\frac{i-2}N\right\}_{i=1}^{2N+3}$ be the uniform mesh on $\big[-1-\frac{1}{N},$ $1+\frac{1}{N}\big]$,
	$L_{\textbf{t}}$ be a linear operator on $C[-B,B]$ given by
	\be
	L_{\textbf t}(g)=\frac{N}{B} \sum_{i=1}^{2N+3} \left({\mathcal L}_N \left(\left\{g(B t_k)\right\}_{k=2}^{2N+2}\right)\right)_i \sigma\left(\cdot-Bt_{i}\right), \qquad g\in C[-B,B].
	\ee
	where ${\mathcal L}_N: \Real^{2N+1} \to \Real^{2N+3}$ is a linear operator that for $\zeta =(\zeta_i)_{i=2}^{2N+2} \in \Real^{2N+1}$
	\begin{equation}\label{diffoperator}
		\left({\mathcal L}_N (\zeta)\right)_i =
		\begin{cases}
			\zeta_2,                            &\text{for}~i=1,\\
			\zeta_3 -2\zeta_2,                &\text{for}~i=2,\\
			\zeta_{i-1} -2\zeta_i +\zeta_{i+1},   &\text{for}~3\leq i \leq 2N+1,\\
			\zeta_{2N+1} -2\zeta_{2N+2},        &\text{for}~i=2N+2,\\
			\zeta_{2N+2},                       &\text{for}~i=2N+3.
		\end{cases}
	\end{equation}
	Then for $g\in C^{0,\alpha}[-B,B]$ with $0<\alpha\leq1$, we have
	$$\left\|L_{\mathbf{t}}(g)-g\right\|_{C\left[-B,B\right]} \leq \frac{2B^\alpha \lvert g \rvert_{0,\alpha}}{N^\alpha},$$
	where $\lvert g \rvert_{0,\alpha}$ is the semi-norm of the Lipschitz-$\alpha$ continuous function $g$.
	\elm
	
	\begin{proof} [Proof of Theorem \ref{thmridge}]
		For the first hidden layer of our FNN, we aim to realize the  ridge functionals $\left\{\int_{\Omega} \sigma(\xi \cdot x-B_\xi t_j)d\mu\right\}_{j}$ leading to approximating the inner functional $L_G^\xi(\mu)$. Take the connection matrix and bias vector as
		\bes
		F^{(1)}= \left[\begin{array}{c}
			\xi^T \\ \xi^T \\ \vdots \\ \xi^T
		\end{array} \right] \in \mathbb{R}^{(2N+3)\times d}, \quad
		b^{(1)}= B_\xi \left[\begin{array}{c}
			t_1 \\ t_2 \\ \vdots \\ t_{2N+3}
		\end{array} \right] \in \mathbb{R}^{2N+3},
		\ees
		then the output of the first hidden layer $h^{(1)}\in \mathbb{R}^{2N+3}$ is
		\be
		\left(h^{(1)}(\mu)\right)_{j}= \int_{\Omega} \sigma\left(\xi \cdot x-B_\xi t_{j}\right)d\mu(x), \quad j=1,2,\dots, 2N+3.
		\ee
		The free parameters in the first hidden layer come only from $\xi$ and $B_\xi t_j$, thus the number of free parameters in this layer is $$\mathcal{N}_1=d+2N+3.$$
		The second hidden layer of our FNN aims to realize the ridge functions  $\left\{\sigma(\ww{\huaS}(\mu)-B_G t_j)\right\}_{j}$ leading to approximating $f\left(\ww{\huaS}(\mu)\right)$, where $\ww{\huaS}(\mu)$ is the approximation of $L_G^\xi(\mu)$ by the linear combination of $h^{(1)}(\mu)$, and $B_G$ is the upper bound of $\left| \ww{\huaS}(\mu)\right|$. From Lemma \ref{splines}, we know that there exists a linear combination of $h^{(1)}(\mu)$ that approximates $L_G^\xi(\mu)$, and denote this approximation as
		\bes
		\ww{\huaS}(\mu):=  \frac{N}{B_\xi} \sum_{j=1}^{2 N+3}\left(\mathcal{L}_{N}\left(\left\{g\left(B_\xi t_{k}\right)\right\}_{k=2}^{2 N+2}\right)\right)_{j} \int_{\Omega}\sigma\left(\xi \cdot x-B_\xi t_{j}\right)d\mu(x),
		\ees
		then according to Lemma \ref{splines}, we have
		\bes \label{ridgeerror1}
		\begin{aligned}
			&\sup_{\mu\in \huaP(\Omega)}\left| \ww{\huaS}(\mu)-L_G^\xi(\mu) \right|\\
			\leq & \sup_{\mu\in \huaP(\Omega)} \left| \frac{N}{B_\xi} \sum_{j=1}^{2 N+3}\left(\mathcal{L}_{N}\left(\left\{g\left(B_\xi t_{k}\right)\right\}_{k=2}^{2 N+2}\right)\right)_{j} \int_{\Omega}\sigma\left(\xi \cdot x-B_\xi t_{j}\right)d\mu(x)-\int_{\Omega}g(\xi \cdot x)d\mu(x) \right|\\
			\leq & \sup_{\mu\in \huaP(\Omega)} \int_{\Omega}\left\|\frac{N}{B_\xi} \sum_{j=1}^{2 N+3}\left(\mathcal{L}_{N}\left(\left\{g\left(B_\xi t_{k}\right)\right\}_{k=2}^{2 N+2}\right)\right)_{j} \sigma\left(\xi \cdot x-B_\xi t_{j}\right)-g(\xi \cdot x)\right\|_{C(\Omega)}d\mu(x) \\
			\leq & \frac{2B_\xi \lvert g\rvert_{0,1}}{N},
		\end{aligned}
		\ees
		which also implies the bound of $\left|\ww{\huaS}(\mu)\right|$,
		\bes
		\sup_{\mu\in \huaP(\Omega)} \left|\ww{\huaS}(\mu)\right| \leq \sup_{\mu\in \huaP(\Omega)}\left| L_G^\xi(\mu)\right|+ \sup_{\mu\in \huaP(\Omega)}\left| \ww{\huaS}(\mu)-L_G^\xi(\mu)\right| \leq  B_g+2B_\xi \lvert g\rvert_{0,1}=B_G.
		\ees
		Then we take the connection matrix and bias vector of the second hidden layer as
		\bes
		F^{(2)}=\frac{N}{B_\xi}\left[\begin{array}{c}
			\mathcal{L}_{N}^T\left(\left\{g\left(B_\xi t_{k}\right)\right\}_{k=2}^{2 N+2}\right)  \\
			\mathcal{L}_{N}^T\left(\left\{g\left(B_\xi t_{k}\right)\right\}_{k=2}^{2 N+2}\right)  \\
			\vdots  \\
			\mathcal{L}_{N}^T\left(\left\{g\left(B_\xi t_{k}\right)\right\}_{k=2}^{2 N+2}\right)
		\end{array}\right] \in \mathbb{R}^{(2N+3)\times(2N+3)}, \ \ \
		b^{(2)}= B_G \left[\begin{array}{c}
			t_1 \\ t_2 \\ \vdots \\ t_{2N+3}
		\end{array}	 \right] \in \mathbb{R}^{2N+3}.
		\ees
		Therefore the output of the second hidden layer $h^{(2)}(\mu) \in \mathbb{R}^{2N+3}$ is
		\be
		\left( h^{(2)}(\mu)\right)_j= \sigma \left( \ww{\huaS}(\mu)-B_G t_j \right)_j, \quad j=1,2,\dots,2N+3.
		\ee
		The free parameters in the second hidden layer come  only from $\frac{N}{B_\xi} \mathcal{L}_{N}\left(\left\{g\left(B_\xi t_{k}\right)\right\}_{k=2}^{2 N+2}\right)$ and $B_G t_j$, thus the number of free parameters in this layer is $$\mathcal{N}_2= 2N+3+2N+3=4N+6.$$
		Finally, choose $c=\frac{N}{B_G} \mathcal{L}_{N}\left(\left\{f\left(B_G t_{k}\right)\right\}_{k=2}^{2 N+2}\right) \in \mathbb{R}^{2N+3}$, according to Lemma \ref{splines}, we have
		\be \label{ridgeerror2}
		\sup_{\mu\in \huaP(\Omega)} \left| c \cdot h^{(2)}(\mu)- f(\ww{\huaS}(\mu)) \right| \leq \frac{2B_G^\beta \left| f \right|_{C^{0,\beta}}}{N^\beta}.
		\ee
		Then we can derive the approximation rate of $f(L_G^\xi(\mu))$ with our constructed FNN by combining (\ref{ridgeerror1}) and (\ref{ridgeerror2}), that is
		\be
		\begin{aligned}
			& \sup_{\mu\in \huaP(\Omega)} \left| c \cdot h^{(2)}(\mu)- f(L_G^\xi(\mu)) \right| \\
			& \leq \sup_{\mu\in \huaP(\Omega)} \left| c \cdot h^{(2)}(\mu)- f(\ww{\huaS}(\mu)) \right|+ \left| f \right|_{C^{0,\beta}}\sup_{\mu\in \huaP(\Omega)} \left( \left|L_G^\xi(\mu)-\ww{\huaS}(\mu) \right| \right)^\beta \\
			& \leq \frac{2B_G^\beta \left| f \right|_{C^{0,\beta}}}{N^\beta}+ \left| f \right|_{C^{0,\beta}}\left(  \f{2B_\xi \lvert g \rvert_{C^{0,1}}}{N} \right)^\beta=\f{2B_G^\beta \left| f \right|_{C^{0,\beta}}+\left| f \right|_{C^{0,\beta}}(2B_\xi \lvert g \rvert_{C^{0,1}})^\beta}{N^\beta}.
		\end{aligned}
		\ee
		From the expression of the coefficient vector $c=\frac{N}{B_G} \mathcal{L}_{N}\left(\left\{f\left(B_G t_{k}\right)\right\}_{k=2}^{2 N+2}\right)$, we have known $\|c\|_{\infty}\leq\f{4\|f\|_{\infty}N}{B_G}$ and hence  \eqref{ridgeapp} holds.
		
		The free parameters of our constructed FNN are from the first, second hidden layer and the coefficients $c$, therefore, the total number of free parameters is
		\be
		\mathcal{N}=\mathcal{N}_1+\mathcal{N}_2+2N+3= 8N+d+12.
		\ee
		This completes the proof of Theorem \ref{thmridge}.
	\end{proof}
	
	\begin{proof}[Proof of Corollary \ref{lapcor}]
		First we define a function
		$$g: [-B_{\xi},B_{\xi}]\rightarrow\mbb R, x\mapsto e^{-x}.$$
		Then we know that $|g|_{C^{0,1}}= e^{B_{\xi}}$ and $B_g=\|g\|_{C[-B_{\xi},B_{\xi}]}= e^{B_{\xi}}$. Then the corresponding $B_G$ in \eqref{BGr} can be explicitly written as
		$$B_G=B_g+2B_{\xi}|g|_{C^{0,1}}= e^{B_{\xi}}+2B_{\xi}e^{B_{\xi}}.$$
		Define a function $f$ as the identity on the interval
		$$[-(e^{B_{\xi}}+2B_{\xi}e^{B_{\xi}}),e^{B_{\xi}}+2B_{\xi}e^{B_{\xi}}]$$
		which satisfies $|f|_{C^{0,1}}=1$ and $\|f\|_{\infty}=e^{B_{\xi}}+2B_{\xi}e^{B_{\xi}}$. We know that $\|f\|_{\infty}=B_G$ and the functional $\huaL_{\xi}$ defined in \eqref{laplace} can be exactly represented by
		$$\huaL_{\xi}(\mu)=f\left(\int_{\Omega}g(\xi\cdot x)d\mu(x)\right)$$
		with $f$, $g$ defined as above. Then the first result follows directly from Theorem \ref{thmridge} by substituting $\beta=1$ and the above bounds.
		
		The second inequality follows directly from Schwarz inequality $|\xi\cdot x|\leq\|\xi\|_2\|x\|_2\leq1$, which implies that $B_{\xi}=\|\xi\cdot x\|_{C(\Omega)}\leq1$.
	\end{proof}

	\subsection{Proof of Theorem \ref{thmpoly}}
	
	\begin{proof} [Proof of Theorem \ref{thmpoly}]
		We first  construct  the two-hidden layer structures in the integral part, which aims to realize an approximation of $g(Q(x))$ for  $x\in \Omega$.
		The first layer of this two-layer structure in the integral part aims to realize the ridge functions $\{\xi_i \cdot x-t_j\}_{i,j}$ leading to approximating the polynomials $\{ (\xi_i \cdot x)^\ell \}_{i,\ell}$. Take the connection matrix as
		\bes
		F^{(1)}= \left[ \begin{array}{c}
			\mathbf{1}_{2N+3} \xi_1^T \\ \mathbf{1}_{2N+3} \xi_2^T \\ \vdots \\\mathbf{1}_{2N+3} \xi_{n_q}^T
		\end{array} \right] \in \mathbb{R}^{n_q(2N+3) \times d},
		\ees
		where $\mathbf{1}_{2N+3}$ is the constant $1$ vector in $\Real^{2N+3}$, and the bias vector $b^{(1)}\in \Real^{n_q(2N+3)}$ as
		\bes
		b^{(1)}_{(k-1)(2N+3)+j}= t_j, \quad k=1,2,\dots,n_q, \ j=1,2,\dots,2N+3,
		\ees
		then we have that
		\bes
		\left( \sigma(F^{(1)}x-b^{(1)}) \right)_{(k-1)(2N+3)+j}= \sigma\left(\xi_k \cdot x-t_j \right), \quad k=1,2,\dots , n_q, \ j=1,2, \dots, 2N+3.
		\ees
		
		The free parameters in this layer are only $\xi_k$ and $t_j$, so the number of free parameters in the first hidden layer is $$\mathcal{N}_1=n_q d+2N+3.$$
		
		The second layer of this two-layer structure in the integral part aims to realize the ridge functions $\left\{\sigma \left( \ww{Q}(x)-B_Q t_j \right) \right\}_{j}$, where $\ww{Q}(x)$ is an approximation of $Q(x)$, and $B_Q$ is the upper bound of $\left| \ww{Q}(x) \right|$ defined explicitly in \eqref{BQ}. Then after the integral with respect to the input distribution $\mu$, we get the output of the integral-part  layer, which is then used to realize an approximation of the inner functional $L_G^Q$.
		
		Denote $v^{[l]}= \mathcal{L}_{N}\left(\left\{t_{j}^{l}\right\}_{j=2}^{2 N+2}\right) \in \Real^{2N+3}$, $\mathbf{V}= \left[v^{[1]} \ v^{[2]} \ \ldots \ v^{[q]}\right]^T \in \Real^{q \times (2N+3)}$, $O$ the zero matrix of size $q \times (2N+3)$ and
		$$
		F^{[N]}=N\left[\begin{array}{cccc}
			\mathbf{V} & O  & \ldots & O  \\
			O  & \mathbf{V} & \ldots & O  \\
			\vdots & \vdots & \vdots & \vdots \\
			O  & O  & \ldots & \mathbf{V}
		\end{array}\right] \in \mathbb{R}^{qn_q \times n_q(2N+3)}.
		$$
		Then we have
		\bes
		F^{[N]}\sigma\left(F^{(1)}x-b^{(1)} \right)= \left[ P_{1}^{[1]}(x) \ \cdots \ P_{1}^{[q]}(x)\  P_{2}^{[1]}(x) \ \cdots \ P_{2}^{[q]}(x) \cdots P_{n_q}^{[1]}(x) \ \cdots \ P_{n_q}^{[q]}(x) \right]^T,
		\ees
		where for $k=1,2,\dots,n_q$ and $\ell=1,2,\dots,q$,
		\bes
		P_{k}^{[\ell]}(x)= N \sum_{j=1}^{2N +3} v^{[\ell]}_j \sigma\left(\xi_{k} \cdot x- t_j\right).
		\ees
		If we denote
		\bes
		\gamma_Q= \left[\gamma_{1,1}^Q \ \gamma_{1,2}^Q \ \ldots \ \gamma_{1,q}^Q \ \gamma_{2,1}^Q \ \gamma_{2,2}^Q \ \ldots \ \gamma_{2,q}^Q \ \ldots \ \gamma_{n_q,1}^Q \ \gamma_{n_q,2}^Q \ \ldots \  \gamma_{n_q,q}^Q \right],
		\ees
		then,
		\bes
		\gamma_Q F^{[N]} \sigma\left(F^{(1)}x-b^{(1)} \right) = \sum_{k =1}^{n_q} \sum_{\ell=1}^q \gamma_{k, \ell}^Q N \sum_{j=1}^{2N +3} v^{[\ell]}_j \sigma\left(\xi_{k} \cdot x -t_{j}\right)
		\ees
		
		According to Lemma \ref{splines}, for $k=1,2,\dots,n_q$, and $\ell=1,2,\dots,q$,
		\be
		\sup_{x \in \Omega} \left|N \sum_{j=1}^{2N +3} v^{[\ell]}_j \sigma\left(\xi_{k} \cdot x -t_{j}\right) -\left(\xi_{k} \cdot x \right)^\ell\right|
		\leq \frac{2\ell}{N},
		\ee
		denote
		$$ \ww{Q}(x)= Q(0)+ \sum_{k =1}^{n_q} \sum_{\ell=1}^q \gamma_{k, \ell}^Q N \sum_{j=1}^{2N +3} v^{[\ell]}_j \sigma\left(\xi_{k} \cdot x -t_{j}\right), $$
		then we have
		\be \label{rateQ}
		\sup_{x \in \Omega} \left|\ww{Q}(x) -Q(x)\right|
		\leq \frac{2q \lVert \gamma_Q \rVert_1}{N},
		\ee
		which also implies the upper bound of $\left|\ww{Q}(x) \right|$ as
		\bes
		\sup_{x \in \Omega} \left|\ww{Q}(x) \right| \leq \sup_{x \in \Omega} \left|Q(x) \right| + \sup_{x \in \Omega} \left|\ww{Q}(x) -Q(x)\right| \leq  \wh{B}_Q+ 2q\lVert \gamma_Q \rVert_1=B_Q.
		\ees
		Take the connection matrix of the second hidden layer as
		\bes
		F^{(2)}= \mathbf{1}_{2N+3}\gamma_Q F^{[N]} \in \Real^{(2N+3)\times n_q(2N+3)},
		\ees
		and the bias vector $b^{(2)}\in \Real^{2N+3}$ as
		\bes
		b^{(2)}_j= -Q(0)+B_Q t_j, \quad \hbox{for} \ j=1,2,\dots,2N+3.
		\ees
		Then the output of the integral-part layer $h^{(2)}\in \Real^{2N+3}$ is
		\be
		\left(h^{(2)}(\mu)\right)_j= \int_{\Omega}\sigma\left( \ww{Q}(x)-B_Q t_j \right)d\mu(x), \quad \hbox{for} \ j=1,2,\dots, 2N+3.
		\ee
		The free parameters in this layer are only $\gamma_Q, Nv^{[\ell]}$ and $-Q(0)+B_Qt_j$, so the number of free parameters in the second hidden layer is $$\mathcal{N}_2=qn_q+q(2N+3)+2N+3=2(q+1)N+qn_q+3q+3.$$
		
		After the realization of the polynomial $Q$, the rest part of the proof is similar as that in the proof of Theorem \ref{thmridge}.
		The third hidden layer of our FNN aims to realize the ridge functions  $\left\{\sigma(\ww{\huaS}(\mu)-B_G t_j)\right\}_{j}$ leading to approximating $f\left(\ww{\huaS}(\mu)\right)$, where $\ww{\huaS}(\mu)$ is the approximation of $L_G^Q(\mu)$ by the linear combination of $h^{(2)}(\mu)$ that is defined in following equation, and $B_G$ is the upper bound of $\left| \ww{\huaS}(\mu)\right|$ defined explicitly in \eqref{BG}. Denote
		\bes
		\ww{\huaS}(\mu):= \frac{N}{B_Q} \sum_{j=1}^{2 N+3}\left(\mathcal{L}_{N}\left(\left\{g\left(B_Q t_{k}\right)\right\}_{k=2}^{2 N+2}\right)\right)_{j} \int_{\Omega}\sigma\left(\ww{Q}(x)-B_Q t_{j}\right)d\mu(x),
		\ees
		then utilizing the error decomposition,
		\bea
		\nono&&\sup_{\mu\in \huaP(\Omega)} \left| \ww{\huaS}(\mu)-L_G^Q(\mu) \right| \\
		\nono&& \leq   \sup_{\mu\in \huaP(\Omega)}\left| \frac{N}{B_Q} \sum_{j=1}^{2 N+3}\left(\mathcal{L}_{N}\left(\left\{g\left(B_Q t_{k}\right)\right\}_{k=2}^{2 N+2}\right)\right)_{j} \int_{\Omega} \sigma\left(\ww{Q}(x)-B_Q t_{j}\right)d\mu-\int_{\Omega}g(\ww{Q}(x))d\mu \right| \\
		\nono&& + \sup_{\mu\in \huaP(\Omega)}\left|\int_{\Omega}g\left(\ww{Q}(x)\right)d\mu- \int_{\Omega}g\left(Q(x)\right)d\mu \right|.
		\eea
		The properties of integral imply that the above terms are bounded by
		\bea
		\nono&& \sup_{\mu\in \huaP(\Omega)} \int_{\Omega}\left\| \frac{N}{B_Q} \sum_{j=1}^{2 N+3}\left(\mathcal{L}_{N}\left(\left\{g\left(B_Q t_{k}\right)\right\}_{k=2}^{2 N+2}\right)\right)_{j} \sigma\left(\ww{Q}(x)-B_Q t_{j}\right)-g\left(\ww{Q}(x)\right)\right\|_{C(\Omega)}d\mu\\
		\nono&& +\sup_{\mu\in \huaP(\Omega)}\int_{\Omega}\left\|g\left(\ww{Q}(x)\right)-g\left(Q(x)\right)\right\|_{C(\Omega)}d\mu.
		\eea
		The fact that $\|\ww{Q}\|_{C(\Omega)}\leq B_Q$, Lemma \ref{splines}, the Lipschitz-$1$ condition of $g$ and (\ref{rateQ}) then imply that
		\be \label{polyerror1}
		\sup_{\mu\in \huaP(\Omega)} \left| \ww{\huaS}(\mu)-L_G^Q(\mu) \right| \leq \frac{2B_Q \lvert g\rvert_{0,1}}{N}+ \lvert g \rvert_{C^{0,1}}\left(\frac{2q\lVert \gamma_Q \rVert_1}{N} \right) \leq \f{3 B_Q \lvert g \rvert_{C^{0,1}}}{N}.
		\ee
		The above procedures also imply the bound of $\left|\ww{\huaS}(\mu)\right|$ as
		\bes
		\sup_{\mu\in \huaP(\Omega)} \left|\ww{\huaS}(\mu)\right| \leq \sup_{\mu\in \huaP(\Omega)}\left| L_G^Q(\mu)\right|+ \sup_{\mu\in \huaP(\Omega)}\left| \ww{\huaS}(\mu)-L_G^Q(\mu)\right| \leq  B_g+3B_Q \lvert g\rvert_{0,1}=B_G.
		\ees
		Then we take the connection matrix and bias vector of the third hidden layer as
		\bes
		F^{(3)}=\frac{N}{B_Q}\left[\begin{array}{c}
			\mathcal{L}_{N}^T\left(\left\{g\left(B_Q t_{k}\right)\right\}_{k=2}^{2 N+2}\right)  \\
			\mathcal{L}_{N}^T\left(\left\{g\left(B_Q t_{k}\right)\right\}_{k=2}^{2 N+2}\right)  \\
			\vdots  \\
			\mathcal{L}_{N}^T\left(\left\{g\left(B_Q t_{k}\right)\right\}_{k=2}^{2 N+2}\right)
		\end{array}\right] \in \mathbb{R}^{(2N+3)\times(2N+3)}, \ \
		b^{(3)}= B_G \left[\begin{array}{c}
			t_1 \\ t_2 \\ \vdots \\ t_{2N+3}
		\end{array}	 \right] \in \mathbb{R}^{2N+3}.
		\ees
		Therefore the output of the third hidden layer $h^{(3)}\in \mathbb{R}^{2N+3}$ is
		\be
		\left( h^{(3)}(\mu)\right)_j= \sigma \left( \ww{\huaS}(\mu)-B_G t_j \right)_j, \quad j=1,2,\dots,2N+3.
		\ee
		The free parameters in the third hidden layer are only $\frac{N}{B_Q} \mathcal{L}_{N}\left(\left\{g\left(B_\xi t_{k}\right)\right\}_{k=2}^{2 N+2}\right)$ and $B_G t_j$, thus the number of free parameters in this layer is $$\mathcal{N}_3= 2N+3+2N+3=4N+6.$$
		
		Finally, choose $c=\frac{N}{B_G} \mathcal{L}_{N}\left(\left\{f\left(B_G t_{k}\right)\right\}_{k=2}^{2 N+2}\right) \in \mathbb{R}^{2N+3}$, according to Lemma \ref{splines}, we have
		\be \label{polyerror2}
		\sup_{\mu\in \huaP(\Omega)} \left| c \cdot h^{(3)}(\mu)- f(\ww{\huaS}(\mu)) \right| \leq \frac{2B_G^\beta \left| f \right|_{C^{0,\beta}}}{N^\beta} .
		\ee
		Then we can derive the approximation rate of our constructed FNN by combining (\ref{polyerror1}) and (\ref{polyerror2}), and using the Lipschitz-$\beta$ condition of $f$, that is
		\be
		\begin{aligned}
			& \sup_{\mu\in \huaP(\Omega)} \left| c \cdot h^{(3)}(\mu)- f(L_G^Q(\mu)) \right| \\
			& \leq \sup_{\mu\in \huaP(\Omega)} \left| c \cdot h^{(3)}(\mu)- f(\ww{\huaS}(\mu)) \right|+ \left| f \right|_{C^{0,\beta}}\sup_{\mu\in \huaP(\Omega)} \left( \left|L_G^Q(\mu)-\ww{\huaS}(\mu) \right| \right)^\beta \\
			& \leq \frac{2B_G^\beta \left| f \right|_{C^{0,\beta}}}{N^\beta}+ \left| f \right|_{C^{0,\beta}}\left( \f{3B_Q \lvert g \rvert_{C^{0,1}}}{N} \right)^\beta=\f{2B_G^\beta \left| f \right|_{C^{0,\beta}}+\left| f \right|_{C^{0,\beta}}(3B_Q \lvert g \rvert_{C^{0,1}})^\beta}{N^\beta}.
		\end{aligned}
		\ee
		From the expression $c=\frac{N}{B_G} \mathcal{L}_{N}\left(\left\{f\left(B_G t_{k}\right)\right\}_{k=2}^{2 N+2}\right)$, we know $\|c\|_{\infty}\leq\f{4\|f\|_{\infty}N}{B_G}$. Then we have \eqref{polyapp}. The free parameters of the whole network are from the first, second, third hidden layer and the coefficient vector $c$, therefore the total number of free parameters is
		\be
		\mathcal{N}=\mathcal{N}_1+\mathcal{N}_2+\mathcal{N}_3+2N+3= (2q+10)N+(d+q)n_q+3q+15.
		\ee
		This proves Theorem \ref{thmpoly}.
	\end{proof}

	\subsection{Compactness of the hypothesis space $\mathcal{H}_{(2,3),R,N}$}
	
	In this section, we rigorously prove the compactness of $\mathcal{H}_{(2,3),R,N}$. The argument is based on a general version of the Ascoli-Arzelà theorem for an abstract framework \cite[Theorem 19.1c]{dd2002}.
	
	We first recall some notations on general equi-continuity and equi-boundedness.
	Let $\left\{f_{n}\right\}$ be a countable collection of continuous functions from a separable topological space $(X, \mathcal{T})$ into a metric space $\left(Y, d_{Y}\right)$. The functions $\{f_{n}\}$ are equibounded at $x$ if the closure in $\left(Y, d_{Y}\right)$ of the set $\left\{f_{n}(x)\right\}$ is compact. The functions $\{f_{n}\}$ are equi-continuous at a point $x \in X$ if for every $\varepsilon>0$, there exists an open set $\mathcal{O} \in \mathcal{T}$ containing $x$ and such that
	$d_{Y}\left(f_{n}(x), f_{n}(y)\right) \leq \varepsilon \quad$ for all $y \in \mathcal{O} \quad$ and all $n \in \mathbb{N}$. Then the general Ascoli-Arzelà theorem can be stated as follow.
	
	\begin{lem}
		Let $\left\{f_{n}\right\}$ be a sequence of continuous functions from a separable topological space $(X, \mathcal{T})$ into a metric space $\left(Y ,d_{Y}\right)$. Assume that the functions $\{f_{n}\}$ are equibounded and equi-continuous at each $x \in X$. Then, there exists a subsequence $\left\{f_{n^{\prime}}\right\} \subset\left\{f_{n}\right\}$ and a continuous function $f: X \rightarrow Y$ such that $\left\{f_{n^{\prime}}\right\} \rightarrow f$ pointwise in $X .$ Moreover the convergence is uniform on compact subsets of $X .$
	\end{lem}
	
	\begin{proof} [Proof of Theorem \ref{cptthm}]
		The idea is to show $\mathcal{H}_{(2,3),R,N}$ is a sequentially compact subset of the metric space $(C(\huaP(\Omega)),\|\cdot\|_{\infty})$. Then $\mathcal{H}_{(2,3),R,N}$ is compact. We have already known from aforementioned argument that the Wasserstein space $(\huaP(\Omega),W_p)$ is a compact separable metric space. Hence, it is naturally a separable topological space with the topology induced by the Wasserstein metric $\huaT=W_p$.
		
		For any countable collection of  functions $\{f_n\}\subset\mathcal{H}_{(2,3),R,N}$ with
		\be
		f_n: (\huaP(\Omega),W_p)\rightarrow(\mbb R,|\cdot|). \label{map}
		\ee
		We first show that $\{f_n\}$ is equi-continuous at any point $\mu\in\huaP(\Omega)$. From the structure of space  $\mathcal{H}_{(2,3),R,N}$, we know that for any function $f_n$ in the collection, there exist $c_{f_n}$, $F_{f_n}^{(j)}$, $b_{f_n}^{(j)}$ with $\|F_{f_n}^{(j)}\|_{\infty}\leq RN^2$, $\|b_{f_n}^{(j)}\|_{\infty}\leq R$, $\|c_{f_n}\|_{\infty}\leq RN$, for $j=1,2,3$ and any $n\in\mbb N$ such that
		\be
		f_n(\mu)=c_{f_n}\cdot\sigma\left(F_{f_n}^{(3)}\int_{\Omega}H_{f_n}^{(2)}(x)d\mu(x)-b_{f_n}^{(3)}\right) \label{fn}
		\ee
		with
		$$H_{f_n}^{(2)}(x)=\sigma\left(F_{f_n}^{(2)}\sigma(F_{f_n}^{(1)}x-b_{f_n}^{(1)})-b_{f_n}^{(2)}\right).$$
		Then  for any $\mu,\nu\in\huaP(\Omega)$ and $n\in\mbb N$,
		\bea
		\nono&&|f_n(\mu)-f_n(\nu)|\\
		\nono&&=\left|c_{f_n}\cdot\left\{\sigma\left(F_{f_n}^{(3)}\int_{\Omega}H_{f_n}^{(2)}(x)d\mu(x)-b_{f_n}^{(3)}\right)-\sigma\left(F_{f_n}^{(3)}\int_{\Omega}H_{f_n}^{(2)}(x)d\nu(x)-b_{f_n}^{(3)}\right)\right\}\right|\\
		&&\leq \|c_{f_n}\|_{1}\|F_{f_n}^{(3)}\|_{\infty}\left\|\int_{\Omega}H_{f_n}^{(2)}(x)d\mu(x)-\int_{\Omega}H_{f_n}^{(2)}(x)d\nu(x)\right\|_{\infty}. \label{aa1}
		\eea
		The definition of the infinity norm for the  vector implies that the above bound equals
		$$\|c_{f_n}\|_{1}\|F_{f_n}^{(3)}\|_{\infty}\max_{1\leq i\leq 2N+3}\left|\int_{\Omega}\left(H_{f_n}^{(2)}(x)\right)_id\mu(x)-\int_{\Omega}\left(H_{f_n}^{(2)}(x)\right)_id\nu(x)\right|$$
		Recall the well-known duality formula for the Kantorovich-Rubinstein metric
		$$
		W_{1}(\mu, \nu)=\sup _{\psi:\|\psi\|_{C^{0,1}} \leq 1}\left\{\int_{\Omega} \psi d \mu-\int_{\Omega} \psi d \nu\right\}.
		$$
		Note that $\|c_{f_n}\|_1\leq(2N+3)\|c_{f_n}\|_{\infty}$, we have \eqref{aa1} is bounded by
		$$(2N+3)\|c_{f_n}\|_{\infty}\|F_{f_n}^{(3)}\|_{\infty}\max_{1\leq i\leq 2N+3}\left\|\left(H_{f_n}^{(2)}\right)_i\right\|_{C^{0,1}}W_1(\mu,\nu).$$
		Also, for any $x,y\in\Omega$ and $n\in\mbb N$, note that for any $1\leq i\leq 2N+3$ there holds
		\bes
		&&\left|\left(H_{f_n}^{(2)}(x)\right)_i-\left(H_{f_n}^{(2)}(y)\right)_i\right|\leq\left\|H_{f_n}^{(2)}(x)-H_{f_n}^{(2)}(y)\right\|_{\infty}\\
		&&=\left\|\sigma\left(F_{f_n}^{(2)}\sigma(F_{f_n}^{(1)}x-b_{f_n}^{(1)})-b_{f_n}^{(2)}\right)-\sigma\left(F_{f_n}^{(2)}\sigma(F_{f_n}^{(1)}y-b_{f_n}^{(1)})-b_{f_n}^{(2)}\right)\right\|_{\infty}\\
		&&\leq\|F_{f_n}^{(2)}\|_{\infty}\left\|F_{f_n}^{(1)}(x-y)\right\|_{\infty}\leq\|F_{f_n}^{(2)}\|_{\infty}\|F_{f_n}^{(1)}\|_{\infty}\|x-y\|_2.
		\ees
		Then it follows that
		$$\max_{1\leq i\leq 2N+3}\left\|\left(H_{f_n}^{(2)}\right)_i\right\|_{C^{0,1}}\leq\|F_{f_n}^{(2)}\|_{\infty}\|F_{f_n}^{(1)}\|_{\infty}$$
		and
		\bes
		|f_n(\mu)-f_n(\nu)|\leq(2N+3)\|c_{f_n}\|_{\infty}\prod_{j=1}^3\|F_{f_n}^{(j)}\|_{\infty}W_1(\mu,\nu)\leq (2N+3)R^4N^7W_1(\mu,\nu), \ \forall \ n\in\mbb N.
		\ees
		We know from \cite{villani}  that there holds $W_1(\mu,\nu)\leq W_p(\mu,\nu)$ for any $p\in[1,\infty)$ which is in fact as a result of H\"older's inequality, then we arrive at
		\be
		|f_n(\mu)-f_n(\nu)|\leq (2N+3) R^4N^7W_p(\mu,\nu), \ \text{for} \ \text{any} \ \mu,\nu\in\huaP(\Omega), n\in\mbb N.
		\ee
		Then we know that, at any point $\mu\in\huaP(\Omega)$, for every $\epsilon>0$, there always exists an open ball in the space $(\huaP(\Omega),W_p)$ centered at $\mu$ in the form of
		\bes
		\huaO_{\mu,\epsilon}=\left\{\nu\in\huaP(\Omega):W_p(\nu,\mu)<\f{\epsilon}{(2N+3)R^4N^7}\right\}
		\ees
		such that $|f_n(\mu)-f_n(\nu)|\leq\epsilon$ for all $\nu\in\huaO_{\mu,\epsilon}$ and all $n\in\mbb N$. That is to say, for any countable collection $\{f_n\}\subset\mathcal{H}_{(2,3),R,N}$, the equi-continuity of $\{f_n\}$ holds at any point of $\huaP(\Omega)$.
		
		For the equi-boundedness of $\{f_n\}$, we only need to show that for any $\mu\in\huaP(\Omega)$ and any $n\in\mbb N$, the function $f_n$ in \eqref{fn} as a map of \eqref{map} is uniformly bounded in $\mbb R$. To this end, since
		$$f_n(\mu)=c_{f_n}\cdot h_{f_n}^{(3)}(\mu)$$
		for some $h_{f_n}^{(3)}$ defined layer by layer as in definition \ref{defFNN}.
		Then for $h_{f_n}^{(1)}(x):= \sigma\left(F_{f_n}^{(1)}x-b_{f_n}^{(1)} \right)$,
		\bes
		\left\Vert h_{f_n}^{(1)}(x) \right\Vert_\infty \leq \lVert F_{f_n}^{(1)} \rVert_\infty \lVert x \rVert_\infty + \lVert b_{f_n}^{(1)} \rVert_\infty \leq RN^2+R \leq 2RN^2, \ \ \  \forall n\in\mbb N.
		\ees
		For $h_{f_n}^{(2)}(\mu)$, we have
		\bes
		\left\Vert h_{f_n}^{(2)}(\mu) \right\Vert_\infty=\left\Vert\int_{\Omega}\sigma\left(F_{f_n}^{(2)}h_{f_n}^{(1)}(x)-b_{f_n}^{(2)}\right)d\mu\right\Vert_{\infty}\leq \int_{\Omega}\left\Vert\sigma\left(F_{f_n}^{(2)}h_{f_n}^{(1)}(x)-b_{f_n}^{(2)}\right)\right\Vert_{\infty} d\mu  \\
		\leq\lVert F_{f_n}^{(2)} \rVert_\infty \left\Vert h_{f_n}^{(1)}(x) \right\Vert_\infty + \lVert b_{f_n}^{(2)} \rVert_\infty \leq (2R^2+R)N^4, \ \forall n\in\mbb N.
		\ees
		Then for the third hidden layer,
		\bes
		\left\Vert h_{f_n}^{(3)}(\mu) \right\Vert_\infty \leq \lVert F_{f_n}^{(3)} \rVert_\infty \left\Vert h_{f_n}^{(2)}(\mu) \right\Vert_\infty + \lVert b_{f_n}^{(3)} \rVert_\infty \leq (2R^3+R^2+R)N^6,  \forall n\in\mbb N.
		\ees
		Finally,
		\be
		|f_n(\mu)|=|c_{f_n}\cdot h_{f_n}^{(3)}(\mu)|\leq\|c_{f_n}\|_{1}\|h_{f_n}^{(3)}(\mu)\|_{\infty}\leq(2N+3)(2R^4+R^3+R^2)N^7,\forall n\in\mbb N.
		\ee
		From the above process, we have in fact shown that $\sup_{\mu\in\huaP(\Omega)}|f_n(\mu)|\leq(2N+3)(2R^4+R^3+R^2)N^7$ holds for all $n\in\mbb N$. Hence the equi-boundedness of $\{f_n\}$ holds at any point of $\huaP(\Omega)$. Combining the above arguments and using the fact that a subset of a metric space is compact if and only if it is a sequentially compact set, we obtain that $\mathcal{H}_{(2,3),R,N}$ is compact.
	\end{proof}

	\subsection{Covering number of space $\mathcal{H}_{(2,3),R,N}$}\label{cnestimate}
	The generalization (estimation) error bound can be derived by a bias and variance trade-off method. The bias corresponds to the approximation rate of our proposed hypothesis space for approximating the target function, which is shown in Theorem \ref{thmpoly}, while we also need to demonstrate the approximation ability of our proposed hypothesis space by showing that the parameters of our constructed FNN can actually be bounded as required in the hypothesis space. The variance term can be measured by the complexity of the hypothesis space, such as pseudo dimension, Rademacher complexity and covering number, and we utilize the covering number as the tool of estimating the complexity. We first bound the parameters in our constructed FNN structure in the following lemma.
	
	\blm \label{boundpara}
	Let $Q$ be a polynomial of degree $q$ on $\Omega$, $g \in C^{0,1}[-B_Q,B_Q]$, $f \in C^{0,\beta}[-B_G,B_G]$, for some $0<\beta \leq 1$. Then for the FNN constructed in the proof of Theorem $\ref{thmpoly}$, there exists a constant $R=R_{d,Q,f,g}$ depending on $d,Q,f,g$ such that for $j=1,2,3$,
	\bes
	\lVert F^{(j)} \rVert_\infty \leq RN^2, \quad \lVert b^{(j)} \rVert_\infty \leq R, \quad \lVert c \rVert_\infty \leq RN.
	\ees
	\elm
	
	\begin{proof}
		For the first hidden layer, since $\left|\xi \right|=1$, and $\left| t_j\right| =1$, for all $j=1,2,\dots,2N+3$, we have that $\lVert F^{(1)} \rVert_\infty \leq \sqrt{d}$ and $\lVert b^{(1)}\rVert_\infty \leq 2$.
		
		For the second hidden layer, since $\lVert v^{(\ell)}\rVert_\infty \leq 4$, for all $\ell=1,2,\dots,q$, and $\lVert Q \rVert_{C(\Omega)}=\wh{B}_Q \leq B_Q$, we have that $\lVert F^{(2)} \rVert_\infty \leq 4\lVert \gamma_{Q} \rVert_1 N(2N+3)$ and $\lVert b^{(2)}\rVert_\infty \leq 3B_Q$.
		
		For the third hidden layer, since $\lVert g \rVert_{C([-B_Q,B_Q])}=B_g \leq B_G$, we have that $\lVert F^{(3)} \rVert_\infty \leq \frac{4B_G}{B_Q}N(2N+3)$ and $\lVert b^{(3)}\rVert_\infty \leq 2B_G$.
		
		For the coefficient vector $c$, we have that $\lVert c \rVert_\infty \leq \frac{4 \lVert f\rVert_\infty}{B_G}N$.
		
		Thus we finish the proof by choosing
		\bes
		R=\max \left\{2\sqrt{d},20 \lVert \gamma_Q \rVert_1, 3B_Q, \frac{20B_G}{B_Q}, 2B_G, \frac{4 \lVert f\rVert_\infty}{B_G} \right\}.
		\ees
	\end{proof}
	
	In the following lemma, we bound the covering number of our hypothesis space $\mathcal{H}_{(2,3),R,N}$, which is then used to derive an upper bound of the estimation error for the ERM algorithm.
	
	\blm \label{coveringnumber}
	For $R\geq 1, N \in \mathbb{N}$, and $\wh{R}=3R^4(10n_q+d+18)$, there exist two constants $T_1, T_2$ depending on $d,q$ such that
	\bes
	\log \mathcal{N}\left(\mathcal{H}_{(2,3),R,N},\epsilon,\lVert \cdot \rVert_\infty\right) \leq T_1N \log{\frac{\wh{R}}{\epsilon}}+ T_2N \log{N}.
	\ees
	\elm
	
	\begin{proof}
		Denote $h^{(1)}(x):= \sigma\left(F^{(1)}x-b^{(1)} \right)$, then
		\bes
		\left\Vert h^{(1)}(x) \right\Vert_\infty \leq \lVert F^{(1)} \rVert_\infty \lVert x \rVert_\infty + \lVert b^{(1)} \rVert_\infty \leq RN^2+R \leq 2RN^2.
		\ees
		For the second layer, we have
		\bes
		\left\Vert h^{(2)}(\mu) \right\Vert_\infty=\left\Vert\int_{\Omega}\sigma\left(F^{(2)}h^{(1)}(x)-b^{(2)}\right)d\mu\right\Vert_{\infty}\leq \sup_{\mu\in\huaP(\Omega)}\int_{\Omega}\left\Vert\sigma\left(F^{(2)}h^{(1)}(x)-b^{(2)}\right)\right\Vert_{\infty} d\mu  \\
		\leq\lVert F^{(2)} \rVert_\infty \left\Vert h^{(1)}(x) \right\Vert_\infty + \lVert b^{(2)} \rVert_\infty \leq 2R^2N^4+R \leq 3R^2N^4.
		\ees
		Then for the third hidden layer,
		\bes
		\left\Vert h^{(3)}(\mu) \right\Vert_\infty \leq \lVert F^{(3)} \rVert_\infty \left\Vert h^{(2)}(\mu) \right\Vert_\infty + \lVert b^{(3)} \rVert_\infty \leq 3R^3N^6+R \leq 4R^3N^6.
		\ees
		
		If we choose another functional $\wh{c} \cdot \wh{h}^{(3)}(\mu)$ in the hypothesis space $\mathcal{H}_{(2,3),R,N}$ induced by $\wh{F}^{(j)}$, $\wh{b}^{(j)}$ and $\wh{c}$, satisfying the restriction that,
		\bes
		\left\vert F^{(j)}_{ik}- \wh{F}^{(j)}_{ik} \right\vert \leq \epsilon, \ \left\Vert b^{(j)}- \wh{b}^{(j)}\right\Vert_\infty \leq \epsilon, \ \hbox{and} \left\Vert c- \wh{c}\right\Vert_\infty \leq \epsilon,
		\ees
		then by the Lipschitz property of ReLU,
		\bes
		\left\Vert h^{(1)}(x) - \wh{h}^{(1)}(x) \right\Vert_\infty \leq \left\lVert F^{(1)}-\wh{F}^{(1)} \right\Vert_\infty \left\Vert x \right\Vert_\infty + \left\lVert b^{(1)}-\wh{b}^{(1)} \right\Vert_\infty \leq (d+1)\epsilon.
		\ees
		Consider the random variable $\xi$, since $\left\vert \mathbb{E}(\xi) \right\vert \leq \mathbb{E} \left(\left\vert \xi \right\vert \right)$, and $\mathbb{E}(\xi)$ is bounded by the maximum value of $\xi$, utilize the Lipschitz property of ReLU,
		\bes
		\begin{aligned}
			&\left\lVert h^{(2)}(\mu)-\wh{h}^{(2)}(\mu) \right\Vert_\infty \leq \left\lVert \int_{\Omega}\sigma\left(F^{(2)}h^{(1)}(x)-b^{(2)}\right)d\mu - \int_{\Omega}\sigma\left(\wh{F}^{(2)}\wh{h}^{(1)}(x)-\wh{b}^{(2)}\right)d\mu \right\Vert_\infty \\
			\leq &  \sup_{\mu\in\huaP(\Omega)} \int_{\Omega}\left\lVert\sigma\left(F^{(2)}h^{(1)}(x)-b^{(2)}\right) - \sigma\left(\wh{F}^{(2)}\wh{h}^{(1)}(x)-\wh{b}^{(2)}\right)\right\Vert_\infty d\mu \\
			\leq & \left\Vert F^{(2)}\left(h^{(1)}(x)-\wh{h}^{(1)}(x)\right) \right\Vert_\infty + \left\Vert \left(F^{(2)}-\wh{F}^{(2)}\right) \wh{h}^{(1)}(x) \right\Vert_\infty + \left\lVert b^{(2)}-\wh{b}^{(2)} \right\Vert_\infty \\
			\leq & RN^2(d+1)\epsilon+ 2RN^2n_q(2N+3)\epsilon+\epsilon \leq (10n_q+d+2)RN^3\epsilon,
		\end{aligned}
		\ees
		where $\|\cdot\|_{\infty}$ in above inequalities obeys the notations at the end of the introduction, specifically, $\|\cdot\|_{\infty}$ on the right hand side of the first inequality is taken w.r.t. vector of functions with variables in $\huaP(\Omega)$. Similarly, for the third hidden layer,
		\bes
		\begin{aligned}
			& \left\lVert h^{(3)}(\mu)-\wh{h}^{(3)}(\mu) \right\Vert_\infty \leq \left\Vert F^{(3)}\left(h^{(2)}(x)-\wh{h}^{(2)}(x)\right) \right\Vert_\infty + \left\Vert \left(F^{(3)}-\wh{F}^{(3)}\right) \wh{h}^{(2)}(x) \right\Vert_\infty \\
			& + \left\lVert b^{(3)}-\wh{b}^{(3)} \right\Vert_\infty  \leq RN^2(10n_q+d+2)RN^3\epsilon+3R^2N^4(2N+3)\epsilon+\epsilon \\
			& \leq (10n_q+d+18)R^2N^5\epsilon.
		\end{aligned}
		\ees
		Finally, the output of these two functionals satisfy
		\be
		\begin{aligned}
			& \left\lVert c \cdot h^{(3)}(\mu)- \wh{c} \cdot \wh{h}^{(3)}(\mu) \right\Vert_\infty \leq \left\lVert c \cdot \left(h^{(3)}(\mu)- \wh{h}^{(3)}(\mu)\right) \right\Vert_\infty + \left\lVert \left(c-\wh{c}\right) \cdot \wh{h}^{(3)}(\mu) \right\Vert_\infty \\
			& \leq (2N+3)RN(10n_q+d+18)R^2N^5\epsilon+(2N+3)4R^3N^6\epsilon \leq 5(10n_q+d+22)R^3N^7\epsilon =:\ww{\epsilon}.
		\end{aligned}
		\ee
		Therefore, by taking a $\epsilon$-net of each free parameter in $F^{(j)}$, $b^{(j)}$ and $c$, the $\ww{\epsilon}$-covering number of $\mathcal{H}_{(2,3),R,N}$ can be bounded by
		\be
		\begin{aligned}
			& \mathcal{N}\left(\mathcal{H}_{(2,3),R,N},\ww{\epsilon},\lVert \cdot \rVert_\infty\right) \leq \left\lceil\frac{2RN^2}{\epsilon}\right\rceil^{n_qd+qn_q+q(2N+3)+2N+3} \left\lceil\frac{2R}{\epsilon}\right\rceil^{6N+9} \left\lceil\frac{2RN}{\epsilon}\right\rceil^{2N+3} \\
			\leq & \left(\frac{3R}{\epsilon}\right)^{(2q+10)N+(d+q)n_q+3q+15} N^{(4q+6)N+2(d+q)n_q+6q+9} \leq \left( \frac{\wh{R}}{\ww{\epsilon}}\right)^{T_1N} N^{T_2N},
		\end{aligned}
		\ee
		where $\wh{R}=15R^4(10n_q+d+18)$, $T_1=(d+q)n_q+5q+5$, and $T_2=9(d+q)n_q+45q+190$. Thus we finish the proof by taking the logarithm.
	\end{proof}
	
	\subsection{Proof of Theorem \ref{oracledr}}
	We first prove the two-stage error decomposition which is crucial for the proof of Theorem \ref{oracledr}.
	\begin{proof} [Proof of Proposition \ref{errordec}]
		With simple computations, we have
		\bes
		\begin{aligned}
			& \mathcal{E}\left(\pi_{M}f_{\hat D,R,N} \right)- \mathcal{E}\left(f_\rho \right)= \mathcal{E}\left(\pi_{M}f_{\hat D,R,N} \right)- \mathcal{E}_{D}\left(\pi_{M}f_{\hat D,R,N} \right)+ \mathcal{E}_D\left(\pi_{M}f_{\hat D,R,N} \right) \\
			& - \mathcal{E}_{\hat{D}}\left(\pi_{M}f_{\hat D,R,N} \right) +\mathcal{E}_{\hat{D}}\left(\pi_{M}f_{\hat D,R,N} \right) - \mathcal{E}_{\hat{D}}\left(h \right) + \mathcal{E}_{\hat{D}}\left(h \right)- \mathcal{E}_{D}\left(h \right)+ \mathcal{E}_{D}\left(h \right) - \mathcal{E}(h) \\
			& + \mathcal{E}(h)- \mathcal{E}(f_\rho).
		\end{aligned}
		\ees
		Then the first decomposition follows from the fact $\mathcal{E}_{\hat{D}}\left(\pi_{M}f_{\hat D,R,N} \right) \leq \mathcal{E}_{\hat{D}}\left(h \right)$. The second further decomposition follows immediately by inserting $\pm \huaE(f_{\rho})$, $\pm \huaE_D(f_{\rho})$ to the above terms.
	\end{proof}

	We  introduce several concentration inequalities based on Bernstein's inequality and Hoeffding's inequality which can be found in or derived from \cite{Cucker2002}, \cite{sz2004} and \cite{Cucker2007}.
	
	\begin{lem} [A-inequality]
		Let $\mathcal{G}$ be a set of continuous functions on a probability space $\mathcal{Z}$ such that, for some $B>0,c>0$, $\left\vert G-\mathbb E(G)\right\vert \leq B$ almost surely and $\mathbb E\left(G^2\right)\leq c\mathbb E(G)$ for all $f\in \mathcal{G}$, then for any $\epsilon>0$ and $0<\alpha\leq1$,
		\bes
		\textit{Prob}\left\{\sup_{G\in \mathcal{G}} \frac{\mathbb E(G)-\frac{1}{m} \sum_{i=1}^m G(z_i)}{\sqrt{\mathbb E(G)+\epsilon}}>4\alpha\sqrt{\epsilon} \right\} \leq \mathcal{N}\left(\mathcal{G},\alpha\epsilon,\Vert \cdot\Vert_\infty \right) \exp \left\{-\frac{\alpha^2m\epsilon}{2c+\frac{2B}{3}} \right\}.
		\ees
	\end{lem}
	
	\begin{lem} [B-inequality]
		Let $\xi$ be a random variable on a probability space $\mathcal{Z}$ with mean $\mathbb E(\xi)$ and variance $\sigma^2(\xi)=\sigma^2$. If $\left\vert \xi(z)-\mathbb E(\xi)\right\vert \leq M_\xi$ for almost all $z\in \mathcal{Z}$, then for any $\epsilon>0$,
		\bes
		\textit{Prob}\left\{ \frac{1}{m} \sum_{i=1}^m \xi(z_i)-\mathbb E(\xi)>\epsilon\right\} \leq \exp \left\{-\frac{m\epsilon^2}{2\left(\sigma^2+\frac{1}{3}M_\xi\epsilon\right)}\right\}.
		\ees
	\end{lem}
	
	\begin{lem} [C-inequality]
		Let $\mathcal{H}_2$ be a set of continuous functions on a probability space $\mathcal{X}$ such that, for some $M_{\huaH_2}>0$, $\left\vert f-\mathbb E(f)\right\vert \leq M_{\huaH_2}$ almost surely for all $f\in \mathcal{H}_2$, then for any $\epsilon>0$,
		\bes
		\textit{Prob}\left\{\sup_{f\in \mathcal{H}_2} \left\vert \mathbb E(f)-\frac{1}{n} \sum_{i=1}^n f(x_i)\right\vert> \epsilon \right\} \leq 2\mathcal{N}\left(\mathcal{H}_2,\frac{\epsilon}{4},\Vert \cdot\Vert_\infty \right) \exp \left\{-\frac{n\epsilon^2}{8M_{\huaH_2}^2} \right\}.
		\ees
	\end{lem}

	Now we prove the oracle inequality for the distribution regression framework based on the  two-stage error decomposition method shown in Proposition \ref{errordec} and these concentration inequalities.
	\begin{proof} [Proof of Theorem \ref{oracledr}]
		We denote $\huaH=\mathcal{H}_{(2,3),R,N}$, the two-stage error decomposition and the basic fact $\left\Vert \pi_{M}f_{\hat D,R,N}-f_\rho \right\Vert_\rho^2= \mathcal{E}\left(\pi_{M}f_{\hat D,R,N} \right)- \mathcal{E}\left(f_\rho \right)$ imply that
		$\left\Vert \pi_{M}f_{\hat D,R,N}-f_\rho \right\Vert_\rho^2$ can be bounded by the sum of  $I_1(D,\mathcal{H})$, $I_2(D,\mathcal{H})$, $\left\vert I_3(\hat{D},\mathcal{H})\right\vert$, $\left\vert I_4(\hat{D},\mathcal{H})\right\vert$ and $R(\mathcal{H})$.
		
		We first derive a bound for $I_1(D,\mathcal{H})$ in a probability form. Consider the functional class
		\bes
		\mathcal{G}:= \left\{G= \left(\pi_{M}f(\mu)-y\right)^2-\left(f_\rho(\mu)-y\right)^2 : f\in\mathcal{H}_{(2,3),R,N} \right\}.
		\ees
		For any fixed $G\in \mathcal{G}$, there exists a $f\in \mathcal{H}_{(2,3),R,N}$ such that $G(z)=\left(\pi_{M}f(\mu)-y\right)^2-\left(f_\rho(\mu)-y\right)^2$, and
		$$\mathbb{E}(G)=\mathcal{E}\left(\pi_{M}f\right)-\mathcal{E}\left(f_\rho\right)=\left\Vert \pi_{M}f-f_\rho\right\Vert_\rho^2,$$
		$$\frac{1}{m} \sum_{i=1}^m G(z_i)=\mathcal{E}_{D}\left(\pi_{M}f\right)-\mathcal{E}_{D}\left(f_\rho\right).$$
		Furthermore, since $\left\vert \pi_{M}f(\mu)\right\vert \leq M$, $\left\vert f_\rho(\mu)\right\vert \leq M$ and $\vert y\vert \leq M$ almost surely, we have
		\bes
		\vert G(z)\vert= \left\vert\left(\pi_{M}f(\mu)-f_\rho(\mu)\right)\left(\pi_{M}f(\mu)+f_\rho(\mu)-2y\right)\right\vert \leq 8M^2.
		\ees
		Thus $\left\vert G(z)-\mathbb{E}(G) \right\vert\leq 16M^2$ and $\mathbb{E}\left(G^2\right)\leq 16M^2 \left\Vert \pi_{M}f-f_\rho\right\Vert_\rho^2=16M^2 \mathbb{E}(G)$. Moreover, since for any $f_1,f_2 \in \mathcal{H}_{(2,3),R,N}$,
		\bes
		\left\vert\left(\pi_{M}f_1(\mu)-y\right)^2-\left(\pi_{M}f_2(\mu)-y\right)^2\right\vert \leq 4M\left\vert\pi_{M}f_1(\mu)-\pi_{M}f_2(\mu)\right\vert \leq 4M\left\vert f_1(\mu)-f_2(\mu)\right\vert,
		\ees
		an $\frac{\epsilon}{4M}$-covering of $\mathcal{H}_{(2,3),R,N}$ results in an $\epsilon$-covering of $\mathcal{G}$, we have
		\bes
		\mathcal{N}\left(\mathcal{G},\epsilon,\Vert \cdot\Vert_\infty \right) \leq \mathcal{N}\left(\mathcal{H}_{(2,3),R,N},\frac{\epsilon}{4M},\Vert \cdot\Vert_\infty \right).
		\ees
		Therefore, by applying A-inequality to $\mathcal{G}$ with $B=c=16M^2$ and $\alpha=\frac{1}{4}$, we have that
		\bes
		\begin{aligned}
			& \textit{Prob}\left\{\sup_{f\in \mathcal{H}_{(2,3),R,N}} \frac{\mathcal{E}\left(\pi_{M}f\right)-\mathcal{E}\left(f_\rho\right)-\left(\mathcal{E}_{D}\left(\pi_{M}f\right)-\mathcal{E}_{D}\left(f_\rho\right) \right)}{\sqrt{\mathcal{E}\left(\pi_{M}f\right)-\mathcal{E}\left(f_\rho\right)+\epsilon}} >\sqrt{\epsilon} \right\} \\
			& \leq  \mathcal{N}\left(\mathcal{G},\frac{\epsilon}{4},\Vert \cdot\Vert_\infty \right) \exp\left\{-\frac{3m\epsilon}{2048M^2} \right\}
			\leq \mathcal{N}\left(\mathcal{H}_{(2,3),R,N},\frac{\epsilon}{16M},\Vert \cdot\Vert_\infty \right) \exp\left\{-\frac{3m\epsilon}{2048M^2} \right\}.
		\end{aligned}
		\ees
		Since $\sqrt{\epsilon\left( \mathcal{E}\left(\pi_{M}f\right)-\mathcal{E}\left(f_\rho\right)+\epsilon\right)}\leq \frac{1}{2}\left(\mathcal{E}\left(\pi_{M}f\right)-\mathcal{E}\left(f_\rho\right)\right)+\epsilon$, by choosing $f=f_{\hat D,R,N}$, we conclude that
		\be \label{I1error}
		\begin{aligned}
			& \textit{Prob}\left\{I_1(D,\mathcal{H})> \frac{1}{2}\left(\mathcal{E}\left(\pi_M f_{\hat D,R,N}\right)-\mathcal{E}\left(f_\rho\right)\right)+\epsilon \right\} \\
			& \leq \mathcal{N}\left(\mathcal{H}_{(2,3),R,N},\frac{\epsilon}{16M},\Vert \cdot\Vert_\infty \right) \exp\left\{-\frac{3m\epsilon}{2048M^2} \right\}.
		\end{aligned}
		\ee
		
		Next we derive a bound for $I_2(D,\mathcal{H})$ in a probability form. Consider the random variable $\xi$ on $\mathcal{Z}$ defined by
		$$\xi(z)=\left(h(\mu)-y\right)^2-\left(f_\rho(\mu)-y\right)^2.$$
		Since $\vert f_\rho(\mu)\vert \leq M$ and $\vert y \vert \leq M$ almost surely, we have
		$$\left\vert \xi(z)\right\vert \leq \left(3M+ \Vert h\Vert_\infty \right)^2, \quad \left\vert \xi-\mathbb{E}(\xi)\right\vert\leq 2\left(3M+ \Vert h\Vert_\infty \right)^2,$$
		$$\sigma^2 \leq \mathbb{E}(\xi^2)\leq \left(3M+ \Vert h\Vert_\infty \right)^2 R(\huaH). $$
		Hence, by applying B-inequality with $M_\xi=2\left(3M+ \Vert h\Vert_\infty \right)^2$, we conclude that
		\be \label{I2error}
		\textit{Prob}\left\{I_2(D,\mathcal{H})>\epsilon \right\} \leq  \exp \left\{-\frac{m\epsilon^2}{2\left(\sigma^2+\frac{1}{3}M_\xi\epsilon\right)}\right\}
		\leq  \exp \left\{-\frac{m\epsilon^2}{2\left(3M+ \Vert h\Vert_\infty \right)^2 \left(R(\mathcal{H})+\frac{2}{3}\epsilon \right)}\right\}.
		\ee
		
		Then we derive a bound for $\left\vert I_3(\hat{D},\mathcal{H})\right\vert$ in a probability form. Since $\left\Vert \pi_{M}f_{\hat D,R,N} \right\Vert_\infty \leq M$ and $\vert y_i \vert \leq M$,
		\bes
		\begin{aligned}
			\left\vert I_3(\hat{D},\mathcal{H})\right\vert= & \left\vert \frac{1}{m} \sum_{i=1}^m \left(\pi_{M}f_{\hat D,R,N}\left(\hat{\mu}_i^n\right)-y_i\right)^2 - \left(\pi_{M}f_{\hat D,R,N}\left(\mu_i\right)-y_i\right)^2 \right\vert \\
			\leq & \frac{1}{m} \sum_{i=1}^m 4M \left\vert f_{\hat D,R,N}\left(\hat{\mu}_i^n\right)- f_{\hat D,R,N}\left(\mu_i\right)\right\vert,
		\end{aligned}
		\ees
		Then we only need to bound $\left\vert f_{\hat D,R,N}\left(\hat{\mu}_i^n\right)- f_{\hat D,R,N}\left(\mu_i\right)\right\vert$. Recall that the structure of the FNN in our hypothesis space $\mathcal{H}_{(2,3),R,N}$ includes the integration w.r.t. the input distribution $\mu$ after the second layer, that is, for any $f\in\mathcal{H}_{(2,3),R,N}$, there always exists a vector of functions
		\be
		H_f^{(2)}(x)=\sigma\left(F_f^{(2)}\sigma(F_f^{(1)}x-b_f^{(1)})-b_f^{(2)}\right)
		\ee
		and   $c_f$, $F_f^{(3)}$, $b_f^{(3)}$ satisfying $\|F_f^{(j)}\|_{\infty}\leq RN^2$, $\|b_f^{(j)}\|_{\infty}\leq R$, $\|c_f\|_{\infty}\leq RN$, $j=1,2,3$ such that
		\bes
		f(\mu)=c_f \cdot \sigma\left(F_f^{(3)} h_f^{(2)}(\mu)-b_f^{(3)}\right)= c_f \cdot \sigma\left(F_f^{(3)} \int_\Omega H_f^{(2)}(x) d\mu(x)-b_f^{(3)}\right).
		\ees
		If we denote a  function class  $\mathcal{H}_2$ by
		\be \label{H_2}
		\begin{aligned}
			\mathcal{H}_{2}= \bigg\{ & H_2(x)= \left(\sigma\left(F^{(2)}\sigma\left(F^{(1)}x-b^{(1)}\right)-b^{(2)}\right)\right)_1:\\
			& \|F^{(j)}\|_{\infty}\leq RN^2, \|b^{(j)}\|_{\infty}\leq R, j=1,2 \bigg\},
		\end{aligned}
		\ee
		which is just the first element of the output a second-layer network in the part of the hypothesis space $\mathcal{H}_{(2,3),R,N}$.
		Therefore, for any $f \in \mathcal{H}_{(2,3),R,N}$, there exists a function $H_{f} \in \mathcal{H}_2$, with
		$$\left(h_f^{(2)}(\mu_i)\right)_1= \mathbb E_{x_{i} \sim\mu_i}\left(H_f \left(x_i\right)\right),\quad i=1,2,...,m,$$ and
		$$\left(h_f^{(2)}(\hat{\mu}_i^n)\right)_1= \frac{1}{n}\sum_{j=1}^n H_f (x_{i,j}), \quad i=1,2,...,m.$$
		
		From the proof of Lemma \ref{coveringnumber} or Lemma \ref{coveringnumberH2} (in Appendix), we know that $\sup_{H_2 \in\huaH_2}\left\vert H_2 \right\vert \leq 3R^2N^4$, thereby $\sup_{H_2\in\huaH_2}\left\vert H_2-\mathbb{E} (H_2) \right\vert \leq 6R^2N^4$. Hence, by applying C-inequality to $\mathcal{H}_2$ with $M_{\huaH_2}= 6R^2N^4$ and using the above notations, we conclude that
		\bes
		\begin{aligned}
			& Prob\left\{\sup_{f \in \mathcal{H}_{(2,3),R,N}} \left| \left(h_f^{(2)}(\mu_i)\right)_1- \left(h_f^{(2)}(\hat{\mu}_i^n)\right)_1 \right|> \epsilon \right\} \\
			= & Prob\left\{\sup_{f \in\mathcal{H}_{(2,3),R,N}} \left| \mathbb E_{x_{i} \sim\mu_i}\left(H_f\left(x_i\right)\right)-\frac{1}{n}\sum_{j=1}^n H_f (x_{i,j}) \right| >\epsilon \right\} \\
			\leq & Prob\left\{\sup_{H_2 \in\huaH_2} \left| \mathbb E_{x_{i} \sim\mu_i}\left(H_2\left(x_i\right)\right)-\frac{1}{n}\sum_{j=1}^n H_2 (x_{i,j}) \right| >\epsilon \right\} \\
			\leq & 2\mathcal{N}\left(\mathcal{H}_2,\frac{\epsilon}{4},\Vert \cdot\Vert_{C(\Omega)} \right) \exp\left\{-\frac{n\epsilon^2}{288R^4N^8}\right\}.
		\end{aligned}
		\ees
		
		From the structure of the hypothesis space $\mathcal{H}_{(2,3),R,N}$, since each row of the connection matrix $F^{(j)}$ and bias $b^{(j)}$ satisfies the same symmetric restriction, we observe from the following explicit form
		\be
		\begin{aligned}
			\nono\sup_{\left\|F^{(j)}\right\|_{\infty}\leq RN^2 \atop \left\|b^{(j) }\right\|_{\infty}\leq R, j=1,2}\Bigg\vert
			& \int_{\Omega}\left(\sigma\left(F^{(2)}\sigma(F^{(1)}x-b^{(1)})-b^{(2)}\right)\right)_sd\mu_i(x) \\
			& -\f{1}{n}\sum_{j=1}^n\left(\sigma\left(F^{(2)}\sigma(F^{(1)}x_{ij}-b^{(1)})-b^{(2)}\right)\right)_s\Bigg\vert
		\end{aligned}
		\ee
		that for any fixed $\mu_i$ and $\hat{\mu}_i=\f{1}{n}\sum_{j=1}^n\delta_{x_{ij}}$, the sup operation w.r.t. $F^{(j)}$ and $b^{(j)}$ contribute equally to any $s$-th row of the $2N+3$ rows. Then we know that
		$$\sup_{f \in \mathcal{H}_{(2,3),R,N}} \left| \left(h_f^{(2)}(\mu_i)\right)_1- \left(h_f^{(2)}(\hat{\mu}_i^n)\right)_1 \right|\leq\epsilon$$
		implies that
		$$\sup_{f \in \mathcal{H}_{(2,3),R,N}} \left| \left(h_f^{(2)}(\mu_i)\right)_s- \left(h_f^{(2)}(\hat{\mu}_i^n)\right)_s \right|\leq\epsilon, \quad \forall s=1,2,\dots,2N+3,$$
		which then implies that $\sup_{f \in \mathcal{H}_{(2,3),R,N}} \left\Vert h_f^{(2)}(\mu_i)-h_f^{(2)}(\hat{\mu}_i^n) \right\Vert_\infty \leq \epsilon$. Then it follows that
		\bes
		\begin{aligned}
			& \sup_{f \in \mathcal{H}_{(2,3),R,N}} \left\vert f(\mu_i)-f(\hat{\mu}_i^n) \right\vert \\
			\leq &  \sup_{f \in \mathcal{H}_{(2,3),R,N}}\Vert c_f \Vert_1 \left\Vert \sigma\left(F_f^{(3)}h_f^{(2)}(\mu_i)-b_f^{(3)}\right) -\sigma\left(F_f^{(3)}h_f^{(2)}(\hat{\mu}_i^n)-b_f^{(3)}\right) \right\Vert_\infty \\
			\leq & (2N+3) \sup_{f \in \mathcal{H}_{(2,3),R,N}}\Vert c_f \Vert_\infty \left\Vert \sigma\left(F_f^{(3)}h_f^{(2)}(\mu_i)-b_f^{(3)}\right) -\sigma\left(F_f^{(3)}h_f^{(2)}(\hat{\mu}_i^n)-b_f^{(3)}\right) \right\Vert_\infty \\
			\leq & (2N+3)RN \Vert F_f^{(3)} \Vert_\infty \sup_{f \in \mathcal{H}_{(2,3),R,N}} \left\Vert h_f^{(2)}(\mu_i)-h_f^{(2)}(\hat{\mu}_i^n) \right\Vert_\infty \leq 5R^2N^4\epsilon.
		\end{aligned}
		\ees
		Thus we have that, for $i=1,2,...,m$,
		\bes
		\textit{Prob}\left\{\sup_{f \in \mathcal{H}_{(2,3),R,N}} \left\vert f(\mu_i)-f(\hat{\mu}_i^n) \right\vert>5R^2N^4\epsilon \right\} \leq 2\mathcal{N}\left(\huaH_2,\frac{\epsilon}{4},\Vert \cdot\Vert_{C(\Omega)}  \right) \exp\left\{-\frac{n\epsilon^2}{288R^4N^8}\right\}.
		\ees
		Now combine all the $m$ terms with $i=1,2,\dots,m$, we have
		\bes
		\begin{aligned}
			& \textit{Prob}\left\{\left\vert I_3(\hat{D},\mathcal{H})\right\vert> 20MR^2N^4\epsilon \right\} \\
			\leq & \textit{Prob}\left\{ \frac{1}{m} \sum_{i=1}^m 4M \left\vert f_{\hat D,R,N}\left(\mu_i\right)- f_{\hat D,R,N}\left(\hat{\mu}_i^n\right)\right\vert > 20MR^2N^4\epsilon \right\} \\
			\leq &  \sum_{i=1}^m\textit{Prob}\left\{ \left\vert f_{\hat D,R,N}(\mu_i)-f_{\hat D,R,N}(\hat{\mu}_i^n) \right\vert>5R^2N^4\epsilon \right\}\\
			\leq & \sum_{i=1}^m\textit{Prob}\left\{\sup_{f \in \mathcal{H}_{(2,3),R,N}} \left\vert f(\mu_i)-f(\hat{\mu}_i^n) \right\vert>5R^2N^4\epsilon \right\} \\
			\leq & 2m \mathcal{N}\left(\huaH_2,\frac{\epsilon}{4},\Vert \cdot\Vert_{C(\Omega)}  \right) \exp\left\{-\frac{n\epsilon^2}{288R^4N^8}\right\}.
		\end{aligned}
		\ees
		Finally, replacing $\epsilon$ by $\frac{\epsilon}{20MR^2N^4}$, we get
		\be \label{I3error}
		\textit{Prob}\left\{\left\vert I_3(\hat{D},\mathcal{H})\right\vert> \epsilon \right\} \leq 2m \mathcal{N}\left(\huaH_2,\frac{\epsilon}{80MR^2N^4},\Vert \cdot\Vert_{C(\Omega)} \right) \exp\left\{-\frac{n\epsilon^2}{115200M^2R^8N^{16}}\right\}.
		\ee
		Similarly, we have
		\be \label{I4error}
		\begin{aligned}
			\textit{Prob}\left\{\left\vert I_4(\hat{D},\mathcal{H})\right\vert> \epsilon \right\} \leq & 2m \mathcal{N}\left(\huaH_2,\frac{\epsilon}{80MR^2N^4},\Vert \cdot\Vert_{C(\Omega)} \right) \\
			& \exp\left\{-\frac{n\epsilon^2}{28800\left(\left\Vert h \right\Vert_{\infty}  +M\right)^2R^8N^{16}}\right\}.
		\end{aligned}
		\ee
		Finally, combining (\ref{I1error}), (\ref{I2error}), (\ref{I3error}) and (\ref{I4error}), we have
		\bes
		\begin{aligned}
			& \textit{Prob}\left\{\Vert \pi_{M}f_{\hat D,R,N}-f_\rho\Vert_\rho^2>2\left\Vert h-f_\rho\right\Vert_\rho^2+8\epsilon \right\} \\
			\leq & \mathcal{N}\left(\mathcal{H}_{(2,3),R,N},\frac{\epsilon}{16M},\Vert \cdot\Vert_\infty \right) \exp\left\{-\frac{3m\epsilon}{2048M^2} \right\} \\
			+& \exp \left\{-\frac{m\epsilon^2}{2\left(3M+ \Vert h\Vert_\infty \right)^2 \left(\left\Vert h-f_\rho\right\Vert_\rho^2+\frac{2}{3}\epsilon \right)}\right\} \\
			+& 4m \mathcal{N}\left(\huaH_2,\frac{\epsilon}{80MR^2N^4},\Vert \cdot\Vert_{C(\Omega)} \right) \exp\left\{-\frac{n\epsilon^2}{115200\max \{\left\Vert h \right\Vert_\infty^2 ,M^2\} R^8N^{16}}\right\}.
		\end{aligned}
		\ees
		Thus we complete the proof by utilizing the covering number bound of the set $\mathcal{H}_{(2,3),R,N}$ and $\huaH_2$ from Lemma \ref{coveringnumber} and Lemma \ref{coveringnumberH2}.
	\end{proof}
	
	\subsection{Proof of Theorem \ref{thmestimation}}
	Based on the oracle inequality for distribution regression in Theorem \ref{oracledr}, we are now ready to give the proof of Theorem \ref{thmestimation}.
	
	\begin{proof} [Proof of Theorem \ref{thmestimation}]
		According to Theorem \ref{thmpoly}, there exists $h^* \in \mathcal{H}_{(2,3),R,N}$, such that
		\bes
		\left\Vert h^*-f_\rho \right\Vert_\rho \leq \sup_{\mu\in\huaP(\Omega)}\left| h^*(\mu)-f_\rho(\mu) \right| \leq C_1 N^{-\beta},
		\ees
		where $C_1= 2B_G^\beta \left| f \right|_{C^{0,\beta}}+(3B_Q \lvert g \rvert_{C^{0,1}})^\beta\left| f \right|_{C^{0,\beta}}$, which also implies that
		\bes
		\left\Vert h^* \right\Vert_\infty \leq \sup_{\mu\in\huaP(\Omega)}\left| h^*(\mu)-f_\rho(\mu) \right|+ \sup_{\mu\in\huaP(\Omega)}\left| f_\rho(\mu) \right| \leq C_1+M.
		\ees
		Then utilizing Theorem \ref{oracledr} by taking $h=h^*$,  it follows that
		\bes
		\begin{aligned}
			& \textit{Prob}\left\{\Vert \pi_{M}f_{\hat D,R,N}-f_\rho\Vert_\rho^2>2C_1^2N^{-2\beta}+8\epsilon \right\} \\
			\leq & \exp\left\{T_1N\log{\frac{16M\wh{R}}{\epsilon}}+T_2N\log{N}-\frac{3m\epsilon}{2048M^2} \right\} \\
			+& \exp \left\{-\frac{m\epsilon^2}{2\left(4M+ C_1 \right)^2 \left(C_1^2 N^{-2\beta}+\frac{2}{3}\epsilon \right)}\right\} \\
			+& \exp\left\{\log{4m}+ T_1N\log{\frac{80M \wh{R} R^2 N^4}{\epsilon}}+T_2N\log{N} -\frac{n\epsilon^2}{115200 \left(M+C_1\right)^2 R^8N^{16}}\right\}.
		\end{aligned}
		\ees
		If we restrict $\epsilon \geq 2C_1^2N^{-2\beta}\log{N}$, we have
		\bes
		\begin{aligned}
			& \textit{Prob}\left\{\Vert \pi_{M}f_{\hat D,R,N}-f_\rho\Vert_\rho^2>9\epsilon \right\} \leq \textit{Prob}\left\{\Vert \pi_{M}f_{\hat D,R,N}-f_\rho\Vert_\rho^2>2C_1^2N^{-2\beta}+8\epsilon \right\} \\
			\leq &  \exp\left\{T_1N\log{\frac{8M\wh{R}N^{2\beta}}{C_1^2}}+T_2N\log{N}-\frac{3m\epsilon}{2048M^2} \right\}
			+ \exp \left\{-\frac{3m\epsilon}{8\left(4M+ C_1 \right)^2}\right\} \\
			+ & \exp\left\{\log{4m}+ T_1N\log{\frac{40M\wh{R} R^2 N^{2\beta+4}}{C_1^2}}+T_2N\log{N} -\frac{n\epsilon^2}{115200 \left(M+C_1\right)^2 R^8N^{16}}\right\} \\
			\leq & \exp\left\{A_1N\log{N}-\frac{3m\epsilon}{2048M^2} \right\}
			+ \exp \left\{-\frac{3m\epsilon}{8\left(4M+ C_1 \right)^2}\right\} \\
			+ & \exp\left\{\log{4m}+ A_2N\log{N} -\frac{n\epsilon^2}{A_3N^{16}}\right\},
		\end{aligned}
		\ees
		where
		\bea
		\nono&&A_1=T_1\left(\log{\frac{8M\wh{R}}{C_1^2}}+2\beta\right)+T_2,\\
		\nono&&A_2=T_1\left(\log{\frac{40M\wh{R} R^2}{C_1^2}}+2\beta+4 \right)+T_2,\\
		\nono&&A_3=115200 \left(M+C_1\right)^2 R^8.
		\eea
		If we choose the neural network parameter $$N=\left[  A_4 m^{\frac{1}{2\beta+1}} \right],$$ with $A_4=\left( \min\left\{ \frac{3C_1^2}{2048M^2A_1},\frac{3C_1^2}{4096M^2A_2}\right\}\right)^{\frac{1}{2\beta+1}}$, and $[x]$ denotes the largest integer that is smaller than or equal to $x$. Moreover, we choose the second stage sample size $$n\geq \left\lceil A_5m^{\frac{4\beta+17}{2\beta+1}} \right\rceil,$$ with $A_5= \frac{3A_3A_4^{2\beta+16}}{4096M^2C_1^2}$, and $\lceil x \rceil$ denotes the smallest integer that is greater than $x$. Then when $m$ satisfying the restriction that
		\be
		\log{4m}\leq \frac{3C_1^2 A_4^{-2\beta}}{4096M^2}m^{\frac{1}{2\beta+1}}=A_6 m^{\frac{1}{2\beta+1}},
		\ee
		we conclude that
		\bes
		\begin{aligned}
			& \textit{Prob}\left\{\Vert \pi_{M}f_{\hat D,R,N}-f_\rho\Vert_\rho^2>9\epsilon \right\} \leq \exp\left\{\frac{3m\epsilon}{4096M^2}-\frac{3m\epsilon}{2048M^2} \right\}
			+ \exp \left\{-\frac{3m\epsilon}{8\left(4M+ C_1 \right)^2}\right\} \\
			& + \exp\left\{\frac{3m\epsilon}{8192M^2}+ \frac{3m\epsilon}{8192M^2} -\frac{3m\epsilon}{2048M^2}\right\} \leq 3\exp \left\{-\frac{3m\epsilon}{256\left(4M+ C_1 \right)^2}\right\},
		\end{aligned}
		\ees
		take $t=9\epsilon$, then when $t\geq 18C_1^2N^{-2\beta}\log N$, we have
		\be \label{esterrprob}
		\textit{Prob}\left\{\Vert \pi_{M}f_{\hat D,R,N}-f_\rho\Vert_\rho^2>t \right\} \leq 3\exp \left\{-\frac{m^{\frac{2\beta}{2\beta+1}}t}{768\left(4M+ C_1 \right)^2}\right\}.
		\ee
		Then using the property that for the non-negative random variable $\xi$, $\mathbb{E}(\xi)=\int_0^\infty P(\xi> t)dt$, with $\xi=\left\Vert \pi_{M} f_{\hat D,R,N}-f_\rho \right\Vert_\rho^2=\mathcal{E}\left(\pi_{M}f_{\hat D,R,N} \right)- \mathcal{E}\left(f_\rho \right)$, from (\ref{esterrprob}), we yield
		\bes
		\begin{aligned}
			&\mathbb{E}\left\{ \mathcal{E}\left(\pi_{M}f_{\hat D,R,N} \right)- \mathcal{E}\left(f_\rho \right) \right\} = \int_{0}^\infty Prob\left\{ \left\Vert \pi_{M} f_{\hat D,R,N}-f_\rho \right\Vert_\rho^2 > t \right\} dt \\
			= & \left(\int_0^{18C_1^2N^{-2\beta}\log N}+ \int_{18C_1^2N^{-2\beta}\log N}^\infty \right) Prob\left\{ \left\Vert \pi_{M} f_{\hat D,R,N}-f_\rho \right\Vert_\rho^2 > t\right\} dt \\
			\leq & 18C_1^2N^{-2\beta}\log N+ \int_{18C_1^2N^{-2\beta}\log N}^\infty 3\exp \left\{-\frac{m^{\frac{2\beta}{2\beta+1}}t}{768\left(4M+ C_1 \right)^2}\right\} \\
			\leq & 18C_1^2 2^{2\beta}A_4^{-2\beta}\left(\log{A_4}+\frac{1}{2\beta+1} \right) m^{-\frac{2\beta}{2\beta+1}} \log m + \int_{0}^\infty 3\exp \left\{-\frac{m^{\frac{2\beta}{2\beta+1}}t}{768\left(4M+ C_1 \right)^2}\right\} \\
			\leq & 18C_1^2 2^{2\beta}A_4^{-2\beta}\left(\log{A_4}+\frac{1}{2\beta+1} \right) m^{-\frac{2\beta}{2\beta+1}} \log m+ 2304\left(4M+ C_1 \right)^2 m^{-\frac{2\beta}{2\beta+1}} \\
			\leq & A_7 m^{-\frac{2\beta}{2\beta+1}} \log m,
		\end{aligned}
		\ees
		where $A_7= 18C_1^2 2^{2\beta}A_4^{-2\beta}\left(\log{A_4}+\frac{1}{2\beta+1} \right)+ 2304\left(4M+ C_1 \right)^2$. Thus we complete the proof of Theorem \ref{thmestimation}.
	\end{proof}
	
	\section*{Acknowledgments}
	The work described in this paper is supported partially by the NSFC/RGC Joint Research Scheme [RGC Project No. N-CityU102/20 and NSFC Project No. 12061160462],
	Germany/Hong Kong Joint Research Scheme [Project No. G-CityU101/20],
	Laboratory for AI-Powered Financial Technologies, Hong Kong Institute for Data Science,
	and the CityU Strategic Interdisciplinary Research Grant [Project No. 7020010].
	
	
	\appendix
	
	\section*{Appendix}
	This appendix provides the proof of the continuity of $f\circ L_G^Q$ and the bound for the covering number of $\huaH_2$.
	\begin{proof} [Proof of Proposition \ref{continuous}]
		It suffices to prove that $L_{G}^Q(\cdot)$ is continuous. We use again the Kantorovich-Rubinstein distance
		$$
		W_{1}(\mu, \nu)=\sup _{\psi:\|\psi\|_{C^{0,1}} \leq 1}\left\{\int_{\Omega} \psi d \mu-\int_{\Omega} \psi d \nu\right\}.
		$$
		Now return to the definition of the functional $L_{G}^Q(\cdot)$ which is defined as
		\be
		\nono L_G^Q(\mu)=\int_{\Omega}g(Q(x))d\mu.
		\ee
		For any polynomial $Q$ on $\Omega$, it is easy to see that for any $x,y\in\Omega$,
		$|Q(x)-Q(y)|\leq \|\nabla Q\|_{C(\Omega)}|x-y|$. Then it follows that $\|g\circ Q\|_{C^{0,1}}\leq |g|_{C^{0,1}}\|\nabla Q\|_{C(\Omega)}$. Combining the above equations yields that
		\be
		\nono \left|L_G^Q(\mu)-L_G^Q(\nu)\right|\leq|g|_{C^{0,1}}\|\nabla Q\|_{C(\Omega)}W_1(\mu,\nu), \forall \mu,\nu\in\huaP(\Omega).
		\ee
		Hölder's inequality shows that, when $p \leq q$, there holds, $W_{p}(\mu,\nu) \leq W_{q}(\mu,\nu),\forall \mu,\nu \in\huaP(\Omega)$. Finally there holds
		\be
		\nono \left|L_G^Q(\mu)-L_G^Q(\nu)\right|\leq|g|_{C^{0,1}}\|\nabla Q\|_{C(\Omega)}W_p(\mu,\nu), \forall \mu,\nu\in\huaP(\Omega),
		\ee
		which shows the continuity of $L_G(\cdot)$.
	\end{proof}
	
	Now we give a bound for the covering number $\mathcal{N}\left(\mathcal{H}_{2},\epsilon,\lVert \cdot \rVert_{C(\Omega)}\right)$
	\blm \label{coveringnumberH2}
	For $R\geq 1, N \in \mathbb{N}$, and $\wh{R}=3R^4(10n_q+d+18)$, there exists two constants $T_1, T_2$ depending on $d,q$ such that
	\bes
	\log \mathcal{N}\left(\mathcal{H}_{2},\epsilon,\lVert \cdot \rVert_{C(\Omega)}\right) \leq T_1N \log{\frac{\wh{R}}{\epsilon}}+ T_2N \log{N}.
	\ees
	\elm
	
	\begin{proof}
		Since $\mathcal{H}_2$ is just the component of a classical two-layer neural network, for the first layer,
		\bes
		\left\Vert \sigma\left(F^{(1)}x-b^{(1)}\right) \right\Vert_\infty \leq \lVert F^{(1)} \rVert_\infty \lVert x \rVert_\infty + \lVert b^{(1)} \rVert_\infty \leq RN^2+R \leq 2RN^2,
		\ees
		for the second layer, we have
		\bes
		\left\Vert H_2(x) \right\Vert_{C(\Omega)} \leq\lVert F^{(2)} \rVert_\infty \left\Vert \sigma\left(F^{(1)}x-b^{(1)}\right) \right\Vert_\infty + \lVert b^{(2)} \rVert_\infty \leq 2R^2N^4+R \leq 3R^2N^4.
		\ees
		If we choose another function $\wh{H}_2 $ in the hypothesis space $\mathcal{H}_{2}$ induced by $\wh{F}^{(j)}$ and $\wh{b}^{(j)}$, satisfying the restriction that,
		\bes
		\left\vert F^{(j)}_{ik}- \wh{F}^{(j)}_{ik} \right\vert \leq \epsilon, \ \left\Vert b^{(j)}- \wh{b}^{(j)}\right\Vert_\infty \leq \epsilon,
		\ees
		then by the Lipschitz property of ReLU,
		\bes
		\begin{aligned}
			\left\Vert \sigma\left(F^{(1)}x-b^{(1)}\right) - \sigma\left(\wh F^{(1)}x-\wh b^{(1)}\right) \right\Vert_\infty \leq & \left\lVert F^{(1)}-\wh{F}^{(1)} \right\Vert_\infty \left\Vert x \right\Vert_\infty + \left\lVert b^{(1)}-\wh{b}^{(1)} \right\Vert_\infty \\
			\leq & (d+1)\epsilon.
		\end{aligned}
		\ees
		It follows that
		\bes
		\begin{aligned}
			\left\lVert H_2(x)-\wh{H}_2(x) \right\Vert_{C(\Omega)}
			\leq & \left\Vert F^{(2)}\left(\sigma\left(F^{(1)}x-b^{(1)}\right)-\sigma\left(\wh F^{(1)}x-\wh b^{(1)}\right)\right) \right\Vert_\infty \\
			+ & \left\Vert \left(F^{(2)}-\wh{F}^{(2)}\right) \sigma\left(\wh F^{(1)}x-\wh b^{(1)}\right) \right\Vert_\infty + \left\lVert b^{(2)}-\wh{b}^{(2)} \right\Vert_\infty \\
			\leq & RN^2(d+1)\epsilon+ 2RN^2n_q(2N+3)\epsilon+\epsilon \leq (10n_q+d+2)RN^3\epsilon =: \tilde \epsilon,
		\end{aligned}
		\ees
		Therefore, by taking a $\epsilon$-net of each free parameter in $F^{(j)}$ and $b^{(j)}$, the $\ww{\epsilon}$-covering number of $\mathcal{H}_{2}$ can be bounded by
		\be
		\begin{aligned}
			\mathcal{N}\left(\mathcal{H}_{2},\ww{\epsilon},\lVert \cdot \rVert_{C(\Omega)}\right) \leq \left\lceil\frac{2RN^2}{\epsilon}\right\rceil^{n_qd+qn_q+q(2N+3)} \left\lceil\frac{2R}{\epsilon}\right\rceil^{4N+6} \leq \left( \frac{\wh{R}}{\ww{\epsilon}}\right)^{T_1N} N^{T_2N},
		\end{aligned}
		\ee
		where $\wh{R}, T_1$ and $T_2$ are  defined in Lemma \ref{coveringnumber}. Thus we finish the proof by taking the logarithm.
	\end{proof}

	\vskip 0.2in
	
	\bibliographystyle{acm}
	\bibliography{biblist.bib}

\begin{thebibliography}{10}

\bibitem{Barron1993}
{\sc Barron, A.~R.}
\newblock Universal approximation bounds for superpositions of a sigmoidal
  function.
\newblock {\em IEEE Transactions on Information Theory 39}, 3 (1993), 930--945.

\bibitem{Bolcskei2019}
{\sc Bolcskei, H., Grohs, P., Kutyniok, G., and Petersen, P.}
\newblock Optimal approximation with sparsely connected deep neural networks.
\newblock {\em SIAM Journal on Mathematics of Data Science 1}, 1 (2019), 8--45.

\bibitem{cs2007}
{\sc Christmann, A., and Steinwart, I.}
\newblock Consistency and robustness of kernel-based regression in convex risk
  minimization.
\newblock {\em Bernoulli 13}, 3 (2007), 799--819.

\bibitem{chui1994}
{\sc Chui, C.~K., Li, X., and Mhaskar, H.~N.}
\newblock Neural networks for localized approximation.
\newblock {\em Mathematics of Computation 63}, 208 (1994), 607--623.

\bibitem{chui2019b}
{\sc Chui, C.~K., Lin, S.-B., and Zhou, D.-X.}
\newblock Deep neural networks for rotation-invariance approximation and
  learning.
\newblock {\em Analysis and Applications 17}, 05 (2019), 737--772.

\bibitem{Cucker2002}
{\sc Cucker, F., and Smale, S.}
\newblock On the mathematical foundations of learning.
\newblock {\em Bulletin of the American Mathematical Society 39\/} (2001),
  1--49.

\bibitem{Cucker2007}
{\sc Cucker, F., and Zhou, D.-X.}
\newblock {\em Learning Theory: An Approximation Theory Viewpoint}, vol.~24.
\newblock Cambridge University Press, 2007.

\bibitem{Cybenko1989}
{\sc Cybenko, G.}
\newblock Approximation by superpositions of a sigmoidal function.
\newblock {\em Mathematics of Control, Signals and Systems 2}, 4 (1989),
  303--314.

\bibitem{Daubechies2022}
{\sc Daubechies, I., DeVore, R., Foucart, S., Hanin, B., and Petrova, G.}
\newblock Nonlinear approximation and (deep) {ReLU} networks.
\newblock {\em Constructive Approximation 55\/} (2022), 127--172.

\bibitem{dd2002}
{\sc DiBenedetto, E., and Debenedetto, E.}
\newblock {\em Real Analysis}.
\newblock Boston: Birkh{\"a}user, 2002.

\bibitem{ds2021}
{\sc Dong, S., and Sun, W.}
\newblock Distributed learning and distribution regression of coefficient
  regularization.
\newblock {\em Journal of Approximation Theory 263\/} (2021), 105523.

\bibitem{ds2022}
{\sc Dong, S., and Sun, W.}
\newblock Learning rate of distribution regression with dependent samples.
\newblock {\em Journal of Complexity\/} (2022), 101679.

\bibitem{fgz2020}
{\sc Fang, Z., Guo, Z.-C., and Zhou, D.-X.}
\newblock Optimal learning rates for distribution regression.
\newblock {\em Journal of Complexity 56\/} (2020), 101426.

\bibitem{fhsys2015}
{\sc Feng, Y., Huang, X., Shi, L., Yang, Y., and Suykens, J. A.~K.}
\newblock Learning with the maximum correntropy criterion induced losses for
  regression.
\newblock {\em The Journal of Machine Learning Research 16}, 30 (2015),
  993--1034.

\bibitem{Goodfellow2016}
{\sc Goodfellow, I., Bengio, Y., and Courville, A.}
\newblock {\em Deep Learning}.
\newblock MIT press, 2016.

\bibitem{guo2020modeling}
{\sc Guo, X., Li, L., and Wu, Q.}
\newblock Modeling interactive components by coordinate kernel polynomial
  models.
\newblock {\em Mathematical Foundations of Computing 3}, 4 (2020), 263.

\bibitem{gs2019}
{\sc Guo, Z.-C., and Shi, L.}
\newblock Optimal rates for coefficient-based regularized regression.
\newblock {\em Applied and Computational Harmonic Analysis 47}, 3 (2019),
  662--701.

\bibitem{Hornik1989}
{\sc Hornik, K., Stinchcombe, M., and White, H.}
\newblock Multilayer feedforward networks are universal approximators.
\newblock {\em Neural Networks 2}, 5 (1989), 359--366.

\bibitem{lin2018distributed}
{\sc Lin, S.-B., and Zhou, D.-X.}
\newblock Distributed kernel-based gradient descent algorithms.
\newblock {\em Constructive Approximation 47}, 2 (2018), 249--276.

\bibitem{LinPinkus1993}
{\sc Lin, V.~Y., and Pinkus, A.}
\newblock Fundamentality of ridge functions.
\newblock {\em Journal of Approximation Theory 75\/} (1993), 295--311.

\bibitem{Lu2021}
{\sc Lu, J., Shen, Z., Yang, H., and Zhang, S.}
\newblock Deep network approximation for smooth functions.
\newblock {\em SIAM Journal on Mathematical Analysis 53}, 5 (2021), 5465--5506.

\bibitem{Maiorov1999a}
{\sc Maiorov, V.~E.}
\newblock On best approximation by ridge functions.
\newblock {\em Journal of Approximation Theory 99}, 1 (1999), 68--94.

\bibitem{msz2021}
{\sc Mao, T., Shi, Z., and Zhou, D.-X.}
\newblock Theory of deep convolutional neural networks iii: Approximating
  radial functions.
\newblock {\em Neural Networks 144\/} (2021), 778--790.

\bibitem{Mhaskar1993}
{\sc Mhaskar, H.~N.}
\newblock Approximation properties of a multilayered feedforward artificial
  neural network.
\newblock {\em Advances in Computational Mathematics 1}, 1 (1993), 61--80.

\bibitem{mn2021}
{\sc M{\"u}cke, N.}
\newblock Stochastic gradient descent meets distribution regression.
\newblock In {\em International Conference on Artificial Intelligence and
  Statistics\/} (2021), PMLR, pp.~2143--2151.

\bibitem{psrw2013}
{\sc P{\'o}czos, B., Singh, A., Rinaldo, A., and Wasserman, L.}
\newblock Distribution-free distribution regression.
\newblock In {\em Artificial Intelligence and Statistics\/} (2013), PMLR,
  pp.~507--515.

\bibitem{rp2001}
{\sc Ray, S., and Page, D.}
\newblock Multiple instance regression.
\newblock In {\em In Proceedings of the 18th International Conference on
  Machine Learning\/} (2001), Morgan Kaufmann, pp.~425--432.

\bibitem{SchmidtHieber2020}
{\sc Schmidt-Hieber, J.}
\newblock Nonparametric regression using deep neural networks with relu
  activation function.
\newblock {\em The Annals of Statistics 48}, 4 (2020), 1875--1897.

\bibitem{sz2004}
{\sc Smale, S., and Zhou, D.-X.}
\newblock Shannon sampling and function reconstruction from point values.
\newblock {\em Bulletin of the American Mathematical Society 41}, 3 (2004),
  279--305.

\bibitem{sz2007}
{\sc Smale, S., and Zhou, D.-X.}
\newblock Learning theory estimates via integral operators and their
  approximations.
\newblock {\em Constructive Approximation 26}, 2 (2007), 153--172.

\bibitem{sc2008}
{\sc Steinwart, I., and Christmann, A.}
\newblock {\em Support Vector Machines}.
\newblock Springer Science \& Business Media, 2008.

\bibitem{sgps2015}
{\sc Szab{\'o}, Z., Gretton, A., P{\'o}czos, B., and Sriperumbudur, B.}
\newblock Two-stage sampled learning theory on distributions.
\newblock In {\em Artificial Intelligence and Statistics\/} (2015), PMLR,
  pp.~948--957.

\bibitem{sspg2016}
{\sc Szab{\'o}, Z., Sriperumbudur, B.~K., P{\'o}czos, B., and Gretton, A.}
\newblock Learning theory for distribution regression.
\newblock {\em The Journal of Machine Learning Research 17}, 1 (2016),
  5272--5311.

\bibitem{villani}
{\sc Villani, C.}
\newblock {\em Optimal Transport: Old and New}, vol.~338.
\newblock Berlin: Springer, 2009.

\bibitem{wyz2006}
{\sc Wu, Q., Ying, Y., and Zhou, D.-X.}
\newblock Learning rates of least-square regularized regression.
\newblock {\em Foundations of Computational Mathematics 6}, 2 (2006), 171--192.

\bibitem{Yarotsky2017}
{\sc Yarotsky, D.}
\newblock Error bounds for approximations with deep relu networks.
\newblock {\em Neural Networks 94\/} (2017), 103--114.

\bibitem{yhsz2021}
{\sc Yu, Z., Ho, D. W.~C., Shi, Z., and Zhou, D.-X.}
\newblock Robust kernel-based distribution regression.
\newblock {\em Inverse Problems 37}, 10 (2021), 105014.

\bibitem{zlt2021}
{\sc Zhang, J., Liu, T., and Tao, D.}
\newblock An optimal transport analysis on generalization in deep learning.
\newblock {\em IEEE Transactions on Neural Networks and Learning Systems\/}
  (2021).

\bibitem{zz2018}
{\sc Zhao, P., and Zhou, Z.-H.}
\newblock Label distribution learning by optimal transport.
\newblock In {\em Proceedings of the AAAI Conference on Artificial
  Intelligence\/} (2018), vol.~32.

\bibitem{zhou2018deepdistributed}
{\sc Zhou, D.-X.}
\newblock Deep distributed convolutional neural networks: Universality.
\newblock {\em Analysis and Applications 16}, 06 (2018), 895--919.

\bibitem{zhou2020acha}
{\sc Zhou, D.-X.}
\newblock Universality of deep convolutional neural networks.
\newblock {\em Applied and Computational Harmonic Analysis 48}, 2 (2020),
  787--794.

\bibitem{Bruna2021}
{\sc Zweig, A., and Bruna, J.}
\newblock A functional perspective on learning symmetric functions with neural
  networks.
\newblock In {\em International Conference on Machine Learning\/} (2021), PMLR,
  pp.~13023--13032.

\end{thebibliography}
	
\end{document}